\documentclass{article}

\usepackage{microtype}
\usepackage{graphicx}
\usepackage{subcaption}
\usepackage{booktabs} 
\usepackage[font=small]{caption}
\usepackage{hyperref}
\makeatletter
\@namedef{ver@algorithmic.sty}{2006/06/02 v1.99 stubbed-out by user}
\makeatother
\usepackage{algorithm}
\usepackage{algpseudocode}

\usepackage[accepted]{icml2026}

\usepackage{amsmath}
\usepackage{amssymb}
\usepackage{mathtools}
\usepackage{amsthm}
\usepackage{enumitem}
\usepackage{float}

\usepackage[capitalize,noabbrev]{cleveref}

\theoremstyle{plain}
\newtheorem{theorem}{Theorem}[section]
\newtheorem{proposition}[theorem]{Proposition}

\theoremstyle{definition}

\theoremstyle{remark}
\newtheorem{remark}[theorem]{Remark}

\usepackage[textsize=tiny]{todonotes}

\icmltitlerunning{Hermite-NGP: Gradient-Augmented Hash Encoding for Learning PDEs}

\begin{document}

\twocolumn[
  \icmltitle{Hermite-NGP: Gradient-Augmented Hash Encoding for Learning PDEs}



  \icmlsetsymbol{equal}{*}

  \begin{icmlauthorlist}
    \icmlauthor{Jinjin He}{gt}
    \icmlauthor{Zhiqi Li}{gt}
    \icmlauthor{Sinan Wang}{gt}
    \icmlauthor{Bo Zhu}{gt}
  \end{icmlauthorlist}

  \icmlaffiliation{gt}{Georgia Institute of Technology, Atlanta, GA, USA}

  \icmlcorrespondingauthor{Jinjin He}{jhe433@gatech.edu}

  \icmlkeywords{Physics-Informed Neural Networks, Hash Encoding, Hermite Interpolation, Partial Differential Equations}

  \vskip 0.3in
]



\printAffiliationsAndNotice{}  

\begin{abstract}
We propose \textit{Hermite-NGP}, a gradient-augmented multi-resolution hash encoding designed to enable fast and accurate computation of spatial derivatives for neural PDE solvers. Unlike existing NGP-based approaches that rely on automatic differentiation or finite differences and suffer from instability or high cost, Hermite-NGP explicitly stores function values and mixed partial derivatives at hash grid vertices, allowing fully analytic evaluation of gradients, Jacobians, and Hessians via Hermite interpolation. This design preserves the efficiency and spatial adaptivity of NGP while supporting analytic differential operators up to second order. We further introduce a multi-resolution curriculum training strategy analogous to multigrid V-cycles to enable coarse-to-fine optimization. Across a range of 2D and 3D PDE benchmarks, Hermite-NGP achieves up to ${\sim}20{\times}$ lower error than prior neural PDE methods, and reduces wall-clock convergence time by $2$--$10\times$ compared to other solvers, with per-epoch training times as low as $3.5\,\mathrm{ms}$ for models with up to $17$M parameters.

\end{abstract}

\section{Introduction}
Efficient neural scene representations that combine classical data structures (grids, tensors, Gaussians) with neural counterparts have emerged as a highly effective family of spatial representations for neural field learning, offering strong locality, spatial adaptivity, and instant training in high-dimensional function approximation. Multi-resolution hash tables, popularized by Instant Neural Graphics Primitives (I-NGP)~\cite{muller2022instant}, are one such design; related approaches include grid- and tensor-decomposition representations~\cite{fridovich2022plenoxels, fridovich2023k, chen2022tensorf, cao2023hexplane, kim2024neuralvdb, chen2025neural, zou2024triplane}, anti-aliased NeRF variants~\cite{barron2023zip}, and Gaussian-based methods such as 3DGS~\cite{kerbl2023gaussian}. These computational merits make NGP-style encodings particularly attractive for representation learning problems that demand both accuracy and efficiency, such as neural radiance field, signed distance field, image reconstruction, and so on ~\cite{chen2023mobilenerf, hu2023tri, chen2023neurbf, li2023neuralangelo, yu2022monosdf, wang2023neus2, muller2022instant}.

However, despite their success in representing graphics primitives, NGP-style encodings prove inadequate for learning PDE-governed functions, especially in the context of physics-informed neural networks (PINNs) that rely on accurate differential constraints ~\cite{raissi2020hidden, karniadakis2021physics, raissi2019physics,raissi2017machine}. \textit{The core challenge lies in I-NGP's inherent inability to compute spatial derivatives efficiently}: standard hash encodings do not support evaluation of first- and second-order derivatives via automatic differentiation, while finite-difference approximations incur prohibitive computational cost and are highly sensitive to the choice of discretization parameters. As a result, naive applications of NGP to neural PDE settings often face challenges in addressing differentiation (see ~\cite{huang2024efficient, li2023neuralangelo, chetan2025accurate, wang2024neural} for examples). This mismatch constitutes a fundamental barrier to leveraging the efficiency of NGP-style representations for solving PDE learning problems.
\begin{figure}[t]
\centering
\captionsetup{belowskip=-16pt}
\includegraphics[width=0.95\linewidth]{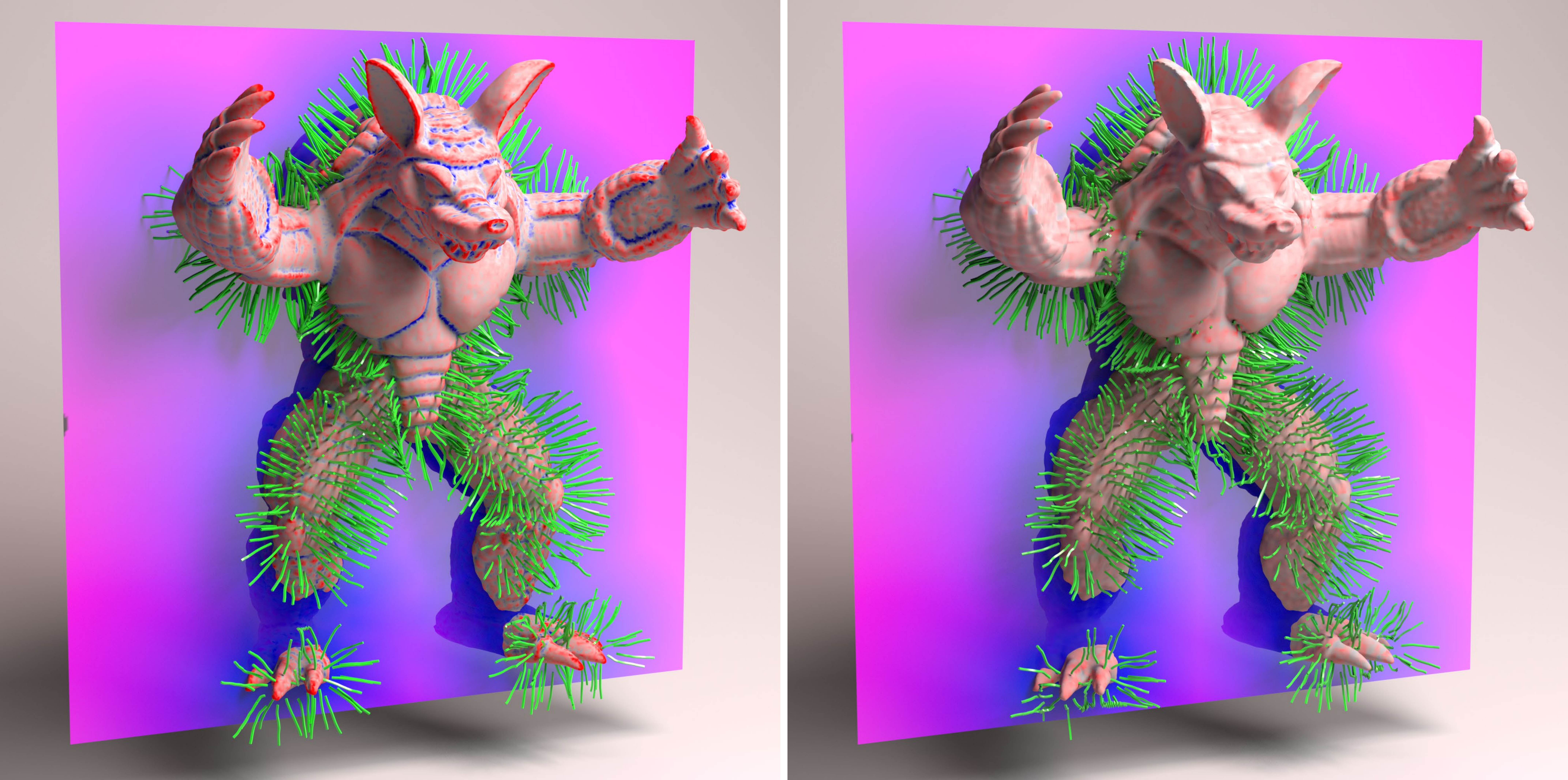}
\caption{\textbf{SDF Gradients and Curvature on Armadillo.} Our method (left) recovers smoother gradient and curvature fields than NeuralAngelo~\cite{li2023neuralangelo} (right), used here as an FD-based hash-encoding baseline under the same SDF+gradient objective. Gradients are visualized as green line segments around the surface, while curvature is shown via mesh coloring.}
\label{fig:teaser}
\end{figure}

To address this challenge of efficient differentiation, we propose \textit{Hermite-NGP}, a gradient-augmented multi-resolution hash encoding designed to support analytic derivative evaluation within a neural representation. Our key insight is inspired by the gradient-augmented field representations (e.g., gradient-augmented level set ~\cite{nave2010gradient}, affine particle-in-cell ~\cite{jiang2015affine}, and flow map methods ~\cite{zhou2024eulerian, deng2023fluid}, where a scalar field discretized on a grid is augmented with an auxiliary, collocated gradient field, enabling continuous and stable derivative reconstruction via Hermite interpolation within each grid cell (see Fig.~\ref{fig:illu1_grad}). Rather than approximating derivatives using finite-difference stencils or recovering them implicitly through automatic differentiation, these methods treat derivative information as a first-class component of the representation itself. By explicitly parameterizing and evolving the gradient field alongside the primary scalar field, gradient-augmented representations yield well-defined, easy-to-calculate differential operators throughout the domain. Such representations have been widely adopted in computational physics and later computer graphics for various simulation problems involving dynamic interfaces \cite{li2023garm, kolomenskiy2016adaptive, kolahdouz2013semi}, material transport \cite{anumolu2018gradient, lee2014narrow}, and vortical structures \cite{bockmann2014gradient, mercier2020characteristic, zhou2024eulerian}.

Hermite-NGP explicitly parameterizes both the hashed feature function and its mixed partial derivatives at each hash grid point. For each spatial resolution, we store Hermite interpolation coefficients corresponding to the function value and its partial derivatives and jointly optimize these quantities during training, resulting in a coupled representation that directly encodes local differential structure. Given these learned coefficients across hash resolutions, Hermite interpolation reconstructs a globally $C^1$-continuous spatial field whose second-order derivatives are analytic within each cell (piecewise across cell boundaries). This representation-level design eliminates truncation error from numerical differencing and avoids instability caused by automatic differentiation through discontinuous hash encodings, enabling analytic evaluation of gradients, Jacobians, Hessians, and Laplacians in a single forward pass while preserving the locality, adaptivity, and scalability of NGP-style encodings (see Figure~\ref{fig:teaser} for an example). We demonstrate the efficacy of our approach across diverse 2D and 3D PDE benchmarks by achieving relative $L^2$ errors down to $10^{-5}$, offering up to $10\times$ accuracy gains over state-of-the-art neural PDE solvers and up to $10^3\times$ improvement over baseline NGP methods, while converging within minutes with per-epoch training costs of $2$--$3.5\,\mathrm{ms}$ on a single GPU.

Our main contributions are:
\begin{itemize}
    \item \textbf{Gradient-Augmented Hash Encoding}.  
    We propose a gradient-augmented hash encoding based on Hermite interpolation that stores mixed partial derivatives at grid vertices, enabling $C^1$ Hermite interpolation with non-trivial second-order structure.

    \vspace{-6pt}

    \item \textbf{Analytic Derivative Evaluation}.  
    The proposed encoding admits fully analytic computation of first- and second-order spatial derivatives, avoiding finite-difference approximations and auto-differentiation.

    \vspace{-6pt}

    \item \textbf{Multi-resolution Coarse-to-Fine Training}.  
    We develop a coarse-to-fine training strategy that leverages the hierarchical structure of Hermite-NGP, inspired by multigrid methods for PDE optimization.

    \vspace{-6pt}

    \item \textbf{Neural PDE Solving with Complex Geometry}.
    We demonstrate Hermite-NGP across diverse 2D/3D PDEs, complex geometric domains, and intrinsic differential operators within a unified learning framework.
\end{itemize}

\vspace{-10pt}

\section{Related Work}

\paragraph{Physics-Informed Neural Networks.}
PINNs~\citep{raissi2017machine,raissi2019physics} embed PDE constraints into networks to obtain
mesh-free solutions~\citep{karniadakis2021physics} with applications in fluid mechanics~\citep{raissi2020hidden} and
libraries like DeepXDE~\citep{lu2021deepxde}. Key challenges include spectral bias~\citep{rahaman2019spectral},
gradient pathologies~\citep{wang2021understanding}, and loss imbalances~\citep{wang2022and}, addressed by adaptive
weighting~\citep{mcclenny2023self}, causal training~\citep{wang2024respecting}, curriculum
strategies~\citep{krishnapriyan2021characterizing,duancopinn}, complex geometry
handling~\citep{costabal2024delta}, and high-dimensional techniques~\citep{hu2024tackling}; recent
benchmarks~\citep{zhongkai2024pinnacle} evaluate these systematically. Alternative representations include Fourier
features~\citep{tancik2020fourier}, PIXEL~\citep{kang2023pixel}, SPINN~\citep{cho2023separable}, and
PIG~\citep{kang2025pig}. Neural operators~\citep{lu2021learning,li2021fourier} learn solution mappings but require supervision.

\vspace{-5pt}

\begin{figure*}[t]
\centering
\captionsetup{belowskip=-14pt}
\includegraphics[width=0.95\linewidth]{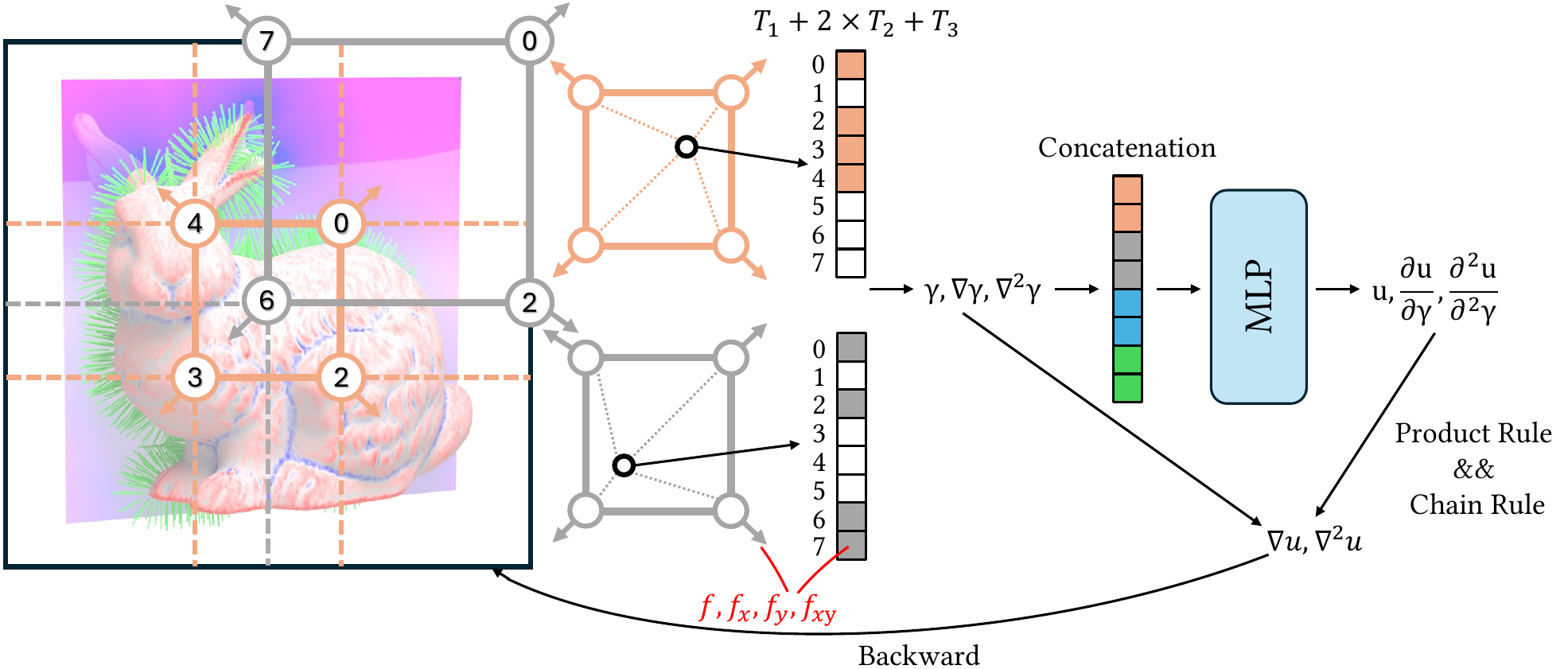}
\caption{\textbf{Hermite-NGP Workflow.} Multi-resolution grids store Hermite coefficients $(f, f_x, f_y, f_{xy})$ in separate hash tables with sizes $T_1$ (values), $T_2$ (first derivatives), and $T_3$ (mixed derivatives), producing
encoding $\gamma, \nabla\gamma, \nabla^2\gamma$. The MLP outputs $u, \partial u/\partial\gamma, \partial^2 u/\partial\gamma^2$, combined via chain rule to yield $\nabla u, \nabla^2 u$ for PDE loss.}
\label{fig:architecture}
\end{figure*}
\paragraph{Multi-Resolution Hash Encoding.}
Instant NGP~\citep{muller2022instant} introduced compact multi-resolution hash tables with $O(1)$ lookup for neural radiance fields~\citep{mildenhall2021nerf}, inspiring extensions such as Zip-NeRF~\citep{barron2023zip}, Neuralangelo~\citep{li2023neuralangelo}, NeuS2~\citep{wang2023neus2}, and Nerfstudio~\citep{tancik2023nerfstudio}. Related grid- and tensor-decomposition representations include TensoRF~\citep{chen2022tensorf}; Gaussian-based methods such as 3D Gaussian Splatting~\citep{kerbl2023gaussian} occupy a distinct but adjacent design space. Recent work also applies hash encodings to efficient PINN training~\citep{huang2024efficient, wang2024neural}.

\vspace{-5pt}

\paragraph{Gradient-Augmented Representation} 
In level set methods, gradient-augmented~\cite{nave2010gradient, bockmann2014gradient} approaches improve interface tracking accuracy by maintaining both signed distance and its gradient~\cite{jiang2015affine}. Similar ideas appear in fluid simulation, where velocity gradients enable higher-order advection
schemes. Recent neural representations have explored gradient supervision for improved surface reconstruction and normal estimation~\cite{sommer2022gradient, huang20242d}. 



\section{Background}
\label{sec:preliminaries}

\paragraph{Multi-Resolution Hash Encoding}
Multi-resolution hash encoding~\citep{muller2022instant} represents a function using $L$ resolution levels of learnable hash tables, each storing $F$-dimensional features.
At level $l \in \{0,\ldots,L-1\}$, an implicit $d$-dimensional grid has resolution
$
N_l = \lfloor N_{\min} b^l \rfloor,
$
where $b = \exp\!\left(\frac{\ln N_{\max}-\ln N_{\min}}{L-1}\right)$.

Given a query point $\mathbf{x}\in\mathbb{R}^d$, the $2^d$ neighboring grid vertices $\{g^{i,l}\}$ are identified.
Their features are retrieved via the hash function $h(g^{i,l}) = \left(\bigoplus_{j=0}^{d-1} g^{i,l}_j \pi_j\right) \bmod T^l$,
where $\bigoplus$ denotes XOR, $\{\pi_j\}$ are fixed primes, and $T^l$ is the table size.
$d$-linear interpolation over these features yields the encoding $\gamma^l(\mathbf{x})$.
Encodings from all levels are concatenated as $\gamma(\mathbf{x})\in\mathbb{R}^{L\cdot F}$ and fed to a lightweight MLP.
However, $d$-linear interpolation is only $C^0$ continuous: first-order derivatives are piecewise constant and discontinuous across cell boundaries, rendering higher-order derivatives undefined.
This precludes analytic evaluation of PDE operators involving first- or second-order derivatives (first derivatives jump across cell boundaries; second derivatives vanish within cells).

\paragraph{Physics-Informed Neural Networks}
Physics-Informed Neural Networks (PINNs)~\citep{raissi2019physics} solve PDEs by embedding physical constraints into neural network training.
Given a PDE $\mathcal{N}_{x,t}[u](x,t) = f(x,t)$ on domain $\Omega \times [0,T]$ with initial condition $u(x,0) = g(x)$ and boundary condition $\mathcal{B}_{x,t}[u] = h(x,t)$ on $\partial\Omega$, PINNs approximate $u$ using a neural network $u_\theta(x,t)$ by minimizing
\begin{equation}
\mathcal{L}(\theta)
= \lambda_{\text{res}}\mathcal{L}_{\text{res}}
+ \lambda_{\text{ic}}\mathcal{L}_{\text{ic}}
+ \lambda_{\text{bc}}\mathcal{L}_{\text{bc}}
+ \lambda_{\text{data}}\mathcal{L}_{\text{data}},
\end{equation}
where $\mathcal{L}_{\text{res}}$ enforces the PDE residual at collocation points, and the remaining terms impose initial, boundary, and data constraints. Derivative boundary conditions (e.g., Neumann) are imposed identically, via a soft loss applied to the corresponding boundary collocation set.

\section{Gradient-Augmented Representation}
\paragraph{Naming Convention.}
We denote the spatial dimension by $d$, the number of resolution levels by $L$, grid resolution at level $l$ by $N_l$, and
grid spacing by $\Delta x_l$. Hash table sizes for derivative type $i$ at level $l$ are $T^l_i$, with feature
dimension $F$ per stored value. Multi-index $\alpha \in \{0,1\}^d$ specifies partial derivative orders,
$\theta_{l,k}^{(\alpha)}$ denotes Hermite coefficients at level $l$ and hash index $k$, and $H^{(\alpha)}$ the
corresponding Hermite basis. The SIREN frequency parameter is $\omega$.

\subsection{Hermite Interpolation}

Hermite interpolation constructs smooth approximations using both function values and derivatives at grid points~\citep{hermite1878formule,birkhoff1968piecewise,ciarlet1991general}.
In 1D, given function values $f_0, f_1$ and derivatives $f'_0, f'_1$ at endpoints of interval $[0,1]$, the cubic Hermite interpolant is:
\begin{equation}
\begin{aligned}
&H(t) = f_0 h^{(0)}(t) + f_1 h^{(0)}(1-t) + f'_0 h^{(1)}(t) + f'_1 h^{(1)}(t),\\
&h^{(0)}(t) = \begin{cases}
-2t^3 + 3t^2, & 0 \le t \le 1,\\
0, & \text{otherwise,}
\end{cases}\hspace{-5mm}, \\&h^{(1)}(t) = \begin{cases}
2t^3 - 3t^2 + 1, & 0 \le t \le 1,\\
0, & \text{otherwise,}
\end{cases}    
\end{aligned}
\label{eq:hermite1d}
\end{equation}
where $h^{(i)}$, $i \in \{0,1\}$, denote the basis functions.  

For a scalar field $f(\mathbf{x})$ in $d$ dimensions defined on a grid with spacing $\Delta x$, at each grid vertex $g$ located at position $\mathbf{x}_g$, we store $2^d$ values, including the function value $f_g = f(\mathbf{x}_g)$ and the mixed partial derivatives $f_g^{(\boldsymbol{\alpha})} =\frac{\partial^{|\mathbf{\alpha}|} f}{(\partial \mathbf{x})_{\mathbf{\alpha}}}
|_{\mathbf{x}=\mathbf{x}_g}, \mathbf{\alpha} \in \{0,1\}^d$, where $\mathbf{\alpha}=(\alpha_i)_{i=0}^{d-1}$ denotes a multi-index and $|\mathbf{\alpha}| = \sum_i \alpha_i$.  For example, in three dimensions, $f_g^{(0,1,1)}=\frac{\partial^2 f}{\partial x_2 \partial x_3}|_{\mathbf{x}=\mathbf{x}_g}$.  Then the Hermite interpolant is constructed via tensor products:
\begin{equation}
H[f](\mathbf{x}) = \sum_{g \in N(\mathbf{x})} \sum_{\alpha \in \{0,1\}^d} f_g^{(\alpha)} \, H^{(\alpha)}\!\left(\frac{\mathbf{x} - \mathbf{x}_g}{\Delta x}\right) \Delta x^{|\alpha|},
\label{eq:hermite_nd}
\end{equation}
where $N(\mathbf{x})$ denotes the $2^d$ vertices of the cell containing $\mathbf{x}$, $H^{(\mathbf{\alpha})}(\mathbf{x}) = \Pi_{i=1^d}h^{(\alpha_i)}(x_i)$ are basis functions for $d$ dimensions.  This construction guarantees $C^1$ continuity with gradients exactly recovered at grid points: $\nabla H[f](\mathbf{x}_g) = (f_g^{(1,0,\ldots)}, f_g^{(0,1,\ldots)}, \ldots)$.
Unlike $d$-linear interpolation where $\nabla^2 u \equiv 0$ inside cells, Hermite interpolation provides well-defined, non-zero second derivatives throughout the domain.

\begin{figure}[t]
\centering
\captionsetup{belowskip=-16pt}
\includegraphics[width=0.95\linewidth]{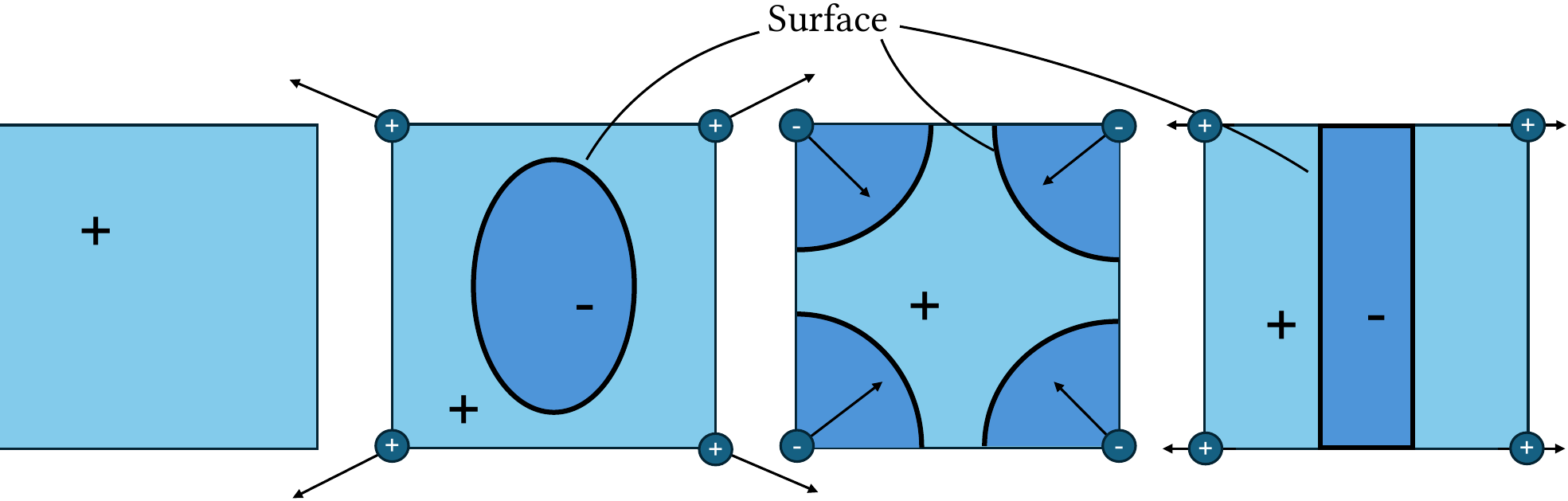}
\caption{Gradient-augmented illustration. The left cell without gradients yields constant interpolation, while the others use gradients to produce rich sub-grid features.}
\label{fig:illu1_grad}
\end{figure}
\vspace{-4pt}

\subsection{Hermite Hash Encoding}
\label{sec:hermite_encoding}

Standard hash encoding stores only function values at grid vertices and relies on $d$-linear interpolation.
While simple and efficient, this yields only $C^0$ continuity: first derivatives are piecewise constant (giving $\nabla^2 u \equiv 0$ inside cells) and discontinuous at cell boundaries.
For PINNs, this fundamentally prevents analytic computation of higher-order derivatives like the Laplacian.
Our key insight is that by storing the complete set of mixed partial derivatives $\partial^\alpha f$ for $\alpha \in \{0,1\}^d$ at each vertex, we have exactly the degrees of freedom needed for $C^1$ Hermite interpolation \citep{ciarlet1991general} with well-defined second derivatives.

For each hash entry, we store the complete set of partial derivatives $\{f^{(\alpha)}\}_{\alpha \in \{0,1\}^d}$---e.g., $(f, f_x, f_y, f_{xy})$ in 2D---requiring $2^d$ coefficients per vertex rather than one.
To manage this increased storage efficiently, we maintain \emph{separate hash tables} grouped by derivative type.
In 2D, this yields three tables: (1) $T_1 \times F$ for function values $f$; (2) $T_2 \times 2F$ for first derivatives $(f_x, f_y)$; and (3) $T_3 \times F$ for the mixed derivative $f_{xy}$.  This separation allows different tables to use different sizes based on their representational requirements.


At resolution level $l$ with grid spacing $\Delta x_l = 1/N_l$, the Hermite hash encoding is:
\begin{equation}
\gamma^l(\mathbf{x}) = \sum_{g \in N_l(\mathbf{x})} \sum_{\alpha \in \{0,1\}^d} \theta_{l,h(g)}^{(\alpha)} H^{(\alpha)}\!\left(\frac{\mathbf{x} - \mathbf{x}_g}{\Delta x_l}\right) \Delta x_l^{|\alpha|},
\label{eq:hermite_encoding}
\end{equation}
where $N_l(\mathbf{x})$ denotes the $2^d$ vertices of the cell containing $\mathbf{x}$ at level $l$, $\theta_{l,h(g)}^{(\alpha)}$ are the learnable Hermite coefficients retrieved from the hash table corresponding to derivative order $|\alpha|$, and $H^{(\alpha)}$ are the basis functions defined in Section~\ref{sec:preliminaries}.
Like standard NGP, hash collisions within each table are resolved implicitly through gradient-based optimization.

Features from all $L$ resolution levels are concatenated to form the final encoding $\gamma(\mathbf{x}) = (\gamma^0(\mathbf{x}), ... ,\gamma^{L-1}(\mathbf{x})) \in \mathbb{R}^{L \cdot F}$, where $F$ is the feature dimension per stored value.  In 2D, the total parameter count across all tables is $L \times (T_1 + 2T_2 + T_3) \times F$.
When all tables use the same size $T$, this simplifies to $L \times 4T \times F$, matching the $2^d$ factor from the 4 Hermite coefficients per vertex.

A key advantage of storing derivative information explicitly is that the encoding captures richer local spatial structure than function values alone.
This enables the network to represent sharp gradients and rapid spatial variations directly in the encoding, which is particularly beneficial for PDEs with complex boundary conditions---the derivative coefficients can adapt locally to enforce boundary constraints while the function values capture the global solution structure.
This spatial adaptivity, inherited from hash encoding's irregular grid structure, allows Hermite-NGP to handle complex geometric boundaries naturally without requiring boundary-conforming meshes. Note that the hash function $h(\cdot)$ is a discrete index lookup mapping integer grid coordinates to table entries: it is not part of the continuous computation graph, and all spatial derivatives are computed through the smooth Hermite basis applied to the retrieved coefficients.

\section{Analytic Differentiation}
\subsection{Analytic Derivative Calculation}
\label{sec:analytic_deriv}

PDE residuals like the Laplacian $\nabla^2 u$ require second-order spatial derivatives.
Standard autodifferentiation cannot provide meaningful second derivatives for hash encoding: the $d$-linear interpolation yields piecewise constant first derivatives, making second derivatives zero almost everywhere.
Standard NGP therefore resorts to finite differences (FD) like INGP-FD~\cite{huang2024efficient, li2023neuralangelo, wang2024neural}, requiring $2d+1$ forward passes (5 in 2D, 7 in 3D) and introducing $O(\epsilon^2)$ truncation error in the estimated second derivatives of the model output.  This truncation error creates a fundamental accuracy ceiling around $10^{-5}$, regardless of how long training continues. For applications requiring higher precision, analytic derivatives are therefore essential rather than merely preferable.
Hermite encoding enables analytic computation in a single forward pass by differentiating the basis functions directly.

Derivatives of the Hermite encoding follow by differentiating the basis functions. Defining $\mathbf{t} = (\mathbf{x}
  - \mathbf{x}_g)/\Delta x_l$, we have:
  \newcommand{\partialk}{\partial^k}

\begin{equation}
\begin{aligned}
\partialk_{i_1\cdots i_k} \gamma_l
&=
\sum_{g \in N_l(\mathbf{x})}
\sum_{\alpha \in \{0,1\}^d}
\theta_{l,h(g)}^{(\alpha)}\,
\Delta x_l^{|\alpha|-k}\,
\partialk_{i_1\cdots i_k} H^{(\alpha)}(\mathbf{t}) .
\end{aligned}
\label{eq:encoding_deriv}
\end{equation}

  For first derivatives ($k=1$) and second derivatives ($k=2$, enabling analytic Laplacian), the scaling $\Delta
  x_l^{|\alpha|-k}$ ensures correct dimensional behavior across resolution levels.

These formulas provide analytic derivatives of the Hermite interpolant within each cell, free of truncation error.

For implementation, the derivatives of the 1D Hermite basis functions are:
\begin{equation}\label{eq:basis_deriv}
\resizebox{\linewidth}{!}{$
\begin{aligned}
\frac{\partial h^{(0)}}{\partial t}(t) &=
\begin{cases}
-6t^2+6t, & 0\le t<1,\\
6t^2+6t,  & -1<t<0,\\
0,        & |t|>1,
\end{cases} \frac{\partial^2 h^{(0)}}{\partial t^2}(t) =
\begin{cases}
-12t+6, & 0\le t<1,\\
12t+6,  & -1<t<0,\\
0,      & |t|\ge1,
\end{cases}
\\
\frac{\partial h^{(1)}}{\partial t}(t) &=
\begin{cases}
3t^2-4t+1, & |t|<1,\\
0,      & |t|\ge1,
\end{cases} \frac{\partial^2 h^{(1)}}{\partial t^2}(t) =
\begin{cases}
6t-4, & |t|<1,\\
0,      & |t|\ge1,
\end{cases}
\end{aligned}
$}
\end{equation}

For the $d$-dimensional Hermite basis, derivatives can be factorized as $\frac{\partial H^{(\alpha)}(\mathbf{y})}{\partial x_i} = \frac{\partial h^{(i)}(y_i)}{\partial x_i}\Pi_{k\neq i} h^{(k)}(y_k)$, $\frac{\partial^2 H^{(\alpha)}(\mathbf{y})}{\partial x_i^2} = \frac{\partial^2 h^{(i)}(y_i)}{\partial x_i^2}\Pi_{k\neq i} h^{(k)}(y_k)$ and $\frac{\partial^2 H^{(\alpha)}(\mathbf{y})}{\partial x_i\partial x_j} = \frac{\partial h^{(i)}(y_i)}{\partial x_i}\frac{\partial h^{(j)}(y_j)}{\partial x_j}\Pi_{k\neq i,j} h^{(k)}(y_k)$ for $i \neq j$, where $\mathbf{y} = (\frac{\mathbf{x}-\mathbf{x}_g}{\Delta x_l})$ (see Appendix~\ref{app:encoding_derivation} for the complete 2D worked example).  This factorization enables efficient vectorized computation: we compute 1D basis values and derivatives once per dimension, then combine them via outer products.
The Laplacian is computed analytically by summing diagonal Hessian entries: $\nabla^2 \gamma = \sum_{i=1}^d \frac{\partial^2 \gamma}{\partial x_i^2}$.


Furthermore, our grouped table structure (Section~\ref{sec:hermite_encoding}) allows allocating smaller tables to higher-order derivatives, providing fine-grained control over the memory-accuracy tradeoff(see Section ~\ref{subsec:hash-ablation} for ablation study).

\vspace{-10pt}
\begin{figure*}[!htb]
    \centering
\captionsetup{belowskip=2pt}

    \includegraphics[width=1.0\linewidth]{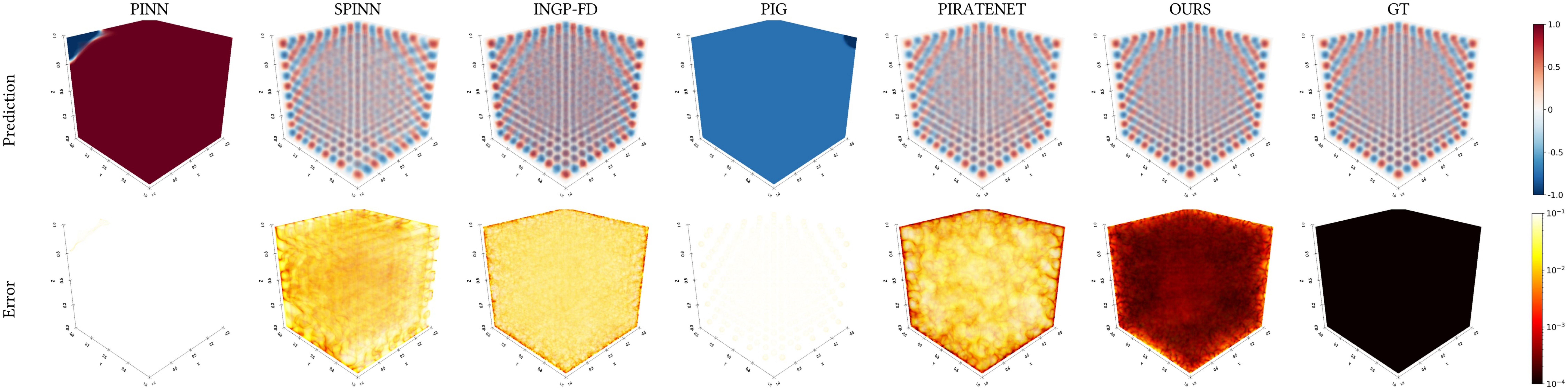}
    \caption{\textbf{Helmholtz 3D ($a{=}10$).} Cross-sectional slices of the 3D Helmholtz solution. Hermite-NGP accurately captures the oscillatory wave patterns, achieving an $L^2$ error of $6\times10^{-3}$, substantially better than the closest baseline, I-NGP-FD ($7.21\times10^{-2}$).}
    \label{fig:helm_3d_a10}
\end{figure*}

\begin{algorithm}
\caption{Hermite-NGP Training}
\label{alg:hermite_ngp}
\begin{algorithmic}[1]
\Require Points $\mathbf{x} \in \mathbb{R}^{N \times d}$, hash tables
$\{\theta_l^{(\alpha)}\}$, MLP $\phi$, active levels $L_{\text{active}}$
\Ensure PDE solution $u(\mathbf{x})$, derivatives $\nabla u$, $\nabla^2 u$
\State \textbf{// Hermite Hash Encoding (Sec.~\ref{sec:hermite_encoding})}
\For{$l = 1, \ldots, L_{\text{active}}$} \Comment{Coarse-to-fine:
Sec.~\ref{sec:Coarse-to-fine}}
  \State $g \gets N_l(\mathbf{x})$ \Comment{$2^d$ grid vertices}
  \State $k \gets h(g) \mod T_l$ \Comment{Hash function}
  \State $\theta_{l,k}^{(\alpha)} \gets \text{Lookup}(\theta_l, k)$ for
  $\alpha \in \{0,1\}^d$
  \State $\gamma_l, \nabla\gamma_l, \nabla^2\gamma_l \gets \sum_g
  \sum_\alpha \theta_{l,k}^{(\alpha)} H^{(\alpha)}(\mathbf{x})$
  \Comment{Eq.~\ref{eq:hermite_encoding}}
\EndFor
\State $\gamma \gets [\gamma_1; \ldots; \gamma_{L_{\text{active}}}]$
\Comment{Concat multi-resolution}
\State \textbf{// MLP with Analytic Derivatives (Sec.~\ref{sec:mlp_deriv})}
\State $u, \frac{\partial u}{\partial \gamma}, \frac{\partial^2
u}{\partial \gamma^2} \gets \text{MLP}_\phi(\gamma)$ \Comment{SIREN}
\State $\nabla u \gets \frac{\partial u}{\partial \gamma} \nabla\gamma$; \quad
$\nabla^2 u \gets \frac{\partial^2 u}{\partial \gamma^2} (\nabla\gamma)^2
+ \frac{\partial u}{\partial \gamma} \nabla^2\gamma$
\State \textbf{// PDE Loss and Backward}
\State $\mathcal{L} \gets \mathcal{L}_{\text{pde}}(u, \nabla u, \nabla^2 u)
+ \lambda_{\text{bc}}\mathcal{L}_{\text{bc}}$
\State $\text{Backward}(\mathcal{L}) \to$ update $\{\theta_l^{(\alpha)}\}, \phi$
\State \Return $u, \nabla u, \nabla^2 u$
\end{algorithmic}
\end{algorithm}

\subsection{End-to-End Differentiation}\label{sec:mlp_deriv}
The complete Hermite-NGP model computes $u_\theta(\mathbf{x}) = \text{MLP}_\phi(\gamma(\mathbf{x}))$, where
  $\gamma(\mathbf{x})$ is the Hermite hash encoding.
  A key advantage is that analytic derivatives from the encoding propagate through the MLP to obtain differential
  operators without finite differences or autograd overhead.

  We use SIREN activations~\citep{sitzmann2020implicit} with $\sigma(x) = \sin(\omega x)$, satisfying $\sigma'' =
  -\omega^2 \sigma$.
  For a single layer $u = W_2 \sin(\omega(W_1 \gamma + b_1)) + b_2$, denoting $z = W_1 \gamma + b_1$ and $a =
  \sin(\omega z)$, the Laplacian is:
  \begin{equation}
  \nabla^2 u = W_2 \Big[ -\omega^2 a \odot \textstyle\sum_{i=1}^d (W_1 \gamma_{x_i})^{2} + \omega \cos(\omega z) \odot
  W_1 \nabla^2 \gamma \Big],
  \label{eq:full_laplacian}
  \end{equation}
  where $\gamma_{x_i} = \frac{\partial \gamma}{\partial x_i}$ and $\nabla^2 \gamma$ are the analytic encoding
  derivatives from Section~\ref{sec:analytic_deriv}.
  All terms except $\gamma_{x_i}$ and $\nabla^2 \gamma$ reuse quantities from the forward pass.

For $K$-layer networks, derivatives propagate recursively. Let $a^{(0)} = \gamma$ and $a^{(k)} = \sin(\omega z^{(k)})$
where $z^{(k)} = W_k a^{(k-1)} + b_k$:
\begin{align}
\frac{\partial a^{(k)}}{\partial x_i} &= \omega \cos(\omega z^{(k)}) \odot W_k \frac{\partial a^{(k-1)}}{\partial
x_i}, \label{eq:recursive_first} \\
\frac{\partial^2 a^{(k)}}{\partial x_i^2} &= -\omega^2 a^{(k)} \odot \left(W_k \tfrac{\partial a^{(k-1)}}{\partial
x_i}\right)^{\!2} \\&+ \omega \cos(\omega z^{(k)}) \odot W_k \frac{\partial^2 a^{(k-1)}}{\partial x_i^2}.
\label{eq:recursive_second}
\end{align}
The base case uses Hermite encoding derivatives; the full Laplacian $\nabla^2 u = \sum_{i=1}^d \frac{\partial^2
u}{\partial x_i^2}$ is computed in a single forward pass (see Appendix~\ref{app:siren_laplacian}).
We initialize MLP weights following~\citet{sitzmann2020implicit} with $\omega_0 = 30$; hash table coefficients start
near zero. Algorithm~\ref{alg:hermite_ngp} summarizes the complete Hermite-NGP training pipeline, from hash encoding through analytic derivatives to PDE loss computation.

The recursion in~\eqref{eq:recursive_first}--\eqref{eq:recursive_second} works with any twice-differentiable activation (Swish, Softplus, GELU, etc.); even with an autograd MLP, the Hermite encoding's analytic derivatives are preserved.

\vspace{-10pt}
\section{Multi-Resolution Coarse-to-Fine Training}\label{sec:Coarse-to-fine}

Inspired by multigrid methods~\citep{briggs2000multigrid} and coarse-to-fine strategies in neural surface
reconstruction~\citep{li2023neuralangelo}, we leverage the hierarchical structure of multi-resolution hash encoding
for curriculum training.
We employ a three-phase strategy:
(1) train only coarse levels ($l = 0, \ldots, L_0$) to capture global structure;
(2) progressively activate finer levels;
(3) fine-tune all levels jointly.
The number of active levels follows $L_{\text{active}}(t) = \min(L, L_0 + \lfloor t / \tau \rfloor)$, where $\tau$ is
the activation interval.
Ablation studies (Section~\ref{sec:ablation}) show coarse-to-fine achieves 79\% error reduction over training all
levels simultaneously. We summarize the pipeline in Algorithm ~\ref{alg:hermite_ngp}.

\vspace{-10pt}
  
  \begin{figure}
    \centering
\captionsetup{belowskip=-12pt}
    \includegraphics[width=1.0\linewidth]{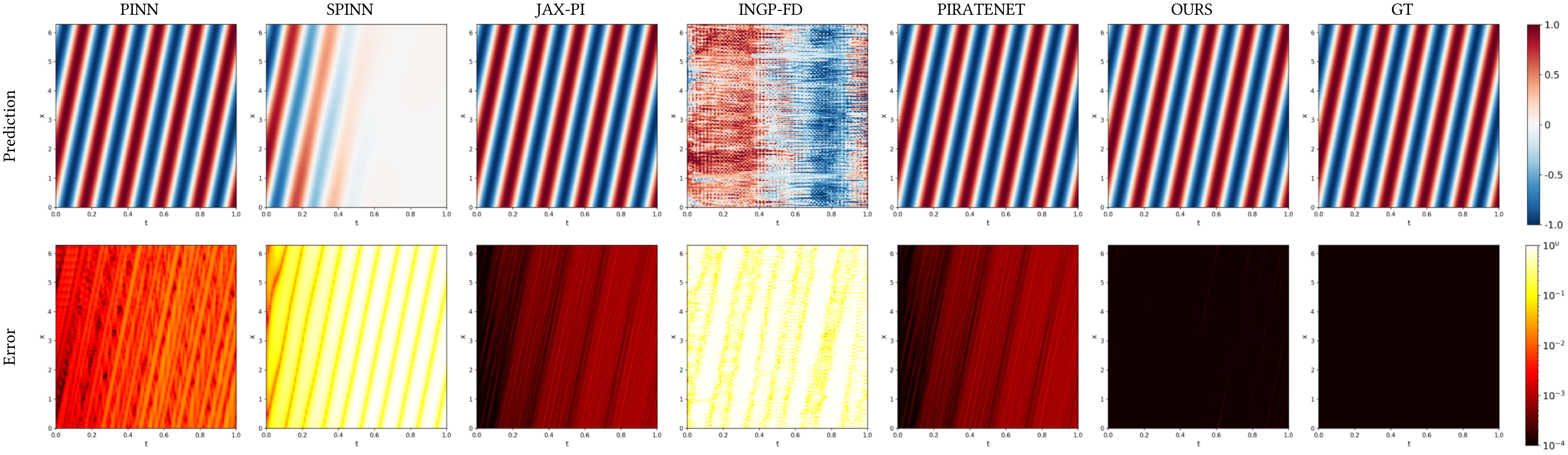}
     \caption{Convection 1+1D ($c=30$): Solution field at final time. Hermite-NGP preserves the sharp traveling wave with
  $L^2 = 8.49{\times}10^{-5}$, a $10\times$ improvement over PirateNet ($8.54\times10^{-4}$), while SPINN and INGP-FD fail to converge.}
    \label{fig:convection}
\end{figure}
\section{Experiments}
\label{sec:experiments}
\begin{figure}
    \centering
\captionsetup{belowskip=-22pt}

   \includegraphics[width=1.0\linewidth]{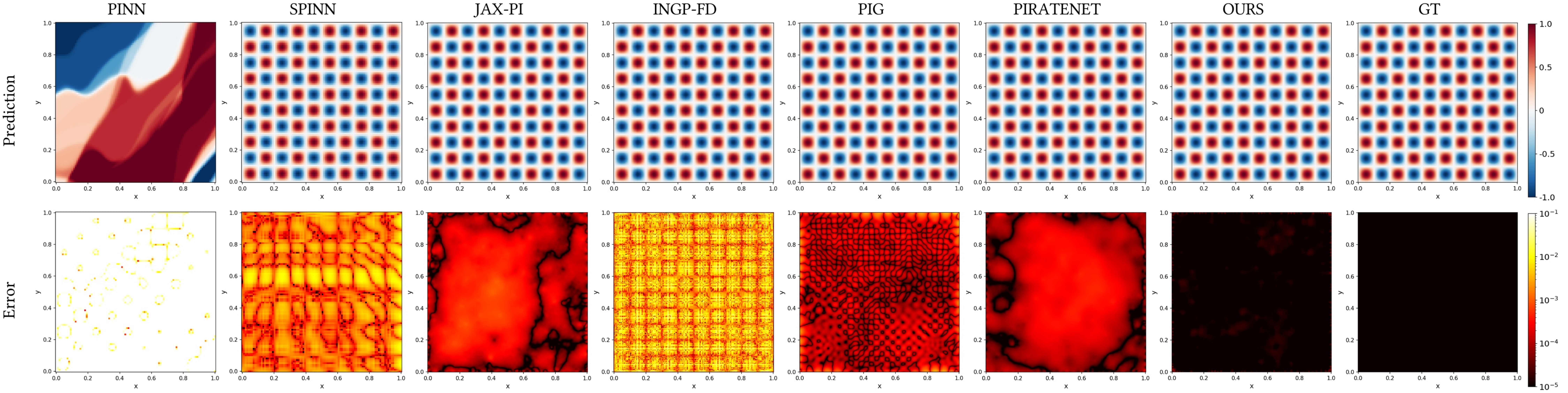}
    \caption{Helmholtz 2D ($a=10$): Ground truth (left), Hermite-NGP prediction (middle), and pointwise error (right). Our
   method achieves relative $L^2$ error of $1.81{\times}10^{-5}$, compared to $3.57{\times}10^{-4}$ for the closest baseline.}
    \label{fig:helm_2d_10}
\end{figure}
We evaluate Hermite-NGP on a broad suite of 2D/3D PDE benchmarks, covering elliptic, hyperbolic, and parabolic equations, as well as geometric applications. These benchmarks probe complementary aspects of neural PDE solvers, including high-frequency oscillations, sharp transport, coupled dynamics, complex geometries, and geometric constraints. For problems without analytic solutions, we obtain ground truth from high-resolution conventional PDE solvers. All metrics are reported as relative $L_2$ errors. Full PDE formulations and training details are provided in App.~\ref{app:exp_details}.

\vspace{-10pt}


\begin{table*}[t]
\centering
\caption{Relative $L^2$ error comparison. Best results in \textbf{bold}. ``fail'' denotes failure ($L^2 \geq 1.0$ or divergence). NA indicates not evaluated. $^\dagger$Reported in the original paper. $^\ddagger$Compact format (larger model). For Helmholtz 2D, our results report small-model performance, with large-model results in parentheses.}

\label{tab:main_results}
\resizebox{\textwidth}{!}{%
\begin{tabular}{l|ccc|cc|ccc}
\toprule
& \multicolumn{3}{c|}{Helmholtz 2D} & \multicolumn{2}{c|}{Helmholtz 3D} & Conv. & T-G & Flow \\
Method & $a{=}10$ & $a{=}20$ & $a{=}100$ & $a{=}3$ & $a{=}10$ & $c{=}30$ & $\nu{=}.01$ & Mix. \\
\midrule
Ours$^\ddagger$ & 5.29e-05 (\textbf{1.81e-05}) & 9.87e-05 (\textbf{7.93e-05}) & \textbf{4.59e-02} & \textbf{6.09e-05} & \textbf{6.01e-03} & \textbf{8.49e-05} & \textbf{7.71e-05} & \textbf{2.35e-04} \\
PirateNet & 3.57e-04 & 1.36e-03 & fail & 8.40e-04 & 1.55e-01 & 8.54e-04 & fail & NA \\
JAX-PI & 5.76e-04 & 1.20e-03 & fail & NA & NA & 8.54e-04 & NA & NA \\
INGP-FD & 1.67e-03 & 2.77e-03 & 4.98e-01 & 4.04e-03 & 7.21e-02 & 7.02e-01 & 6.80e-01 & fail \\
SPINN & 5.46e-03 & 3.30e-02 & 7.08e-01 & 2.28e-02 & 9.61e-02 & 7.92e-01 & 3.98e-01 & 2.90e-03$^\dagger$ \\
PIXEL & 3.47e-02 & 1.34e-01 & fail & NA & NA & 1.84e-03$^\dagger$ & NA & NA \\
PINN & fail & fail & fail & 2.58e-01 & fail & 1.11e-02 & 4.57e-02 & NA \\
PIG & 7.04e-04 & fail & fail & 2.57e-01 & fail & NA & 7.27e-04 & 2.67e-04$^\dagger$ \\
\bottomrule
\end{tabular}%
}
\end{table*}

\begin{figure}
    \centering
\captionsetup{belowskip=-20pt}

    \includegraphics[width=1.0\linewidth]{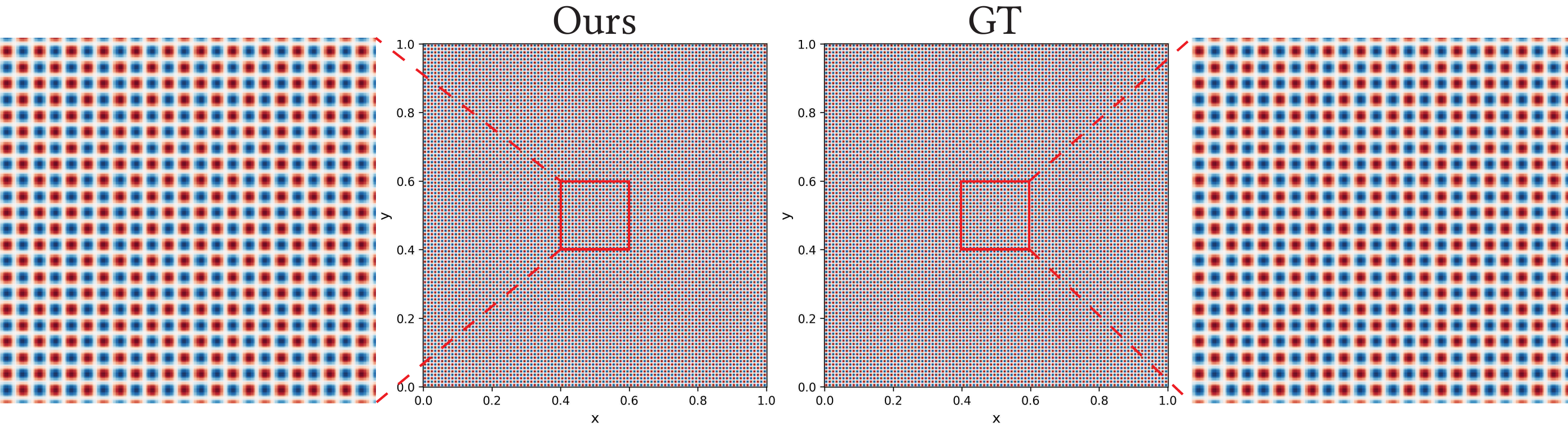}
    \caption{Helmholtz 2D ($a=100$) enlarged: Hermite-NGP is the only method that converges on this challenging
  high-frequency setting, achieving $L^2 = 4.59{\times}10^{-2}$. All baseline methods fail to capture the rapid
  oscillations.}
    \label{fig:helm_2d_single_ours_large}
\end{figure}

\subsection{Experimental Setup}
\label{sec:exp_setup}

\paragraph{Baselines.}
We compare Hermite-NGP against recent PINN-based methods: PirateNet~\citep{wang2024piratenets}, JAX-PI~\citep{wang2023expert}, PIG~\citep{kang2025pig}, INGP-FD~\citep{huang2024efficient} (hash encoding with finite differences), SPINN~\citep{cho2023separable}, PIXEL~\citep{kang2023pixel}, and a vanilla PINN baseline~\citep{raissi2019physics}. For all baselines, we use the official implementations with recommended hyperparameters and, unless otherwise noted, a shared training protocol (Adam~\cite{kingma2014adam} optimizer, GradNorm loss balancing), see details in Appendix ~\ref{app:training_config}.

Table~\ref{tab:main_results} summarizes the relative $L_2$ errors across all benchmarks, and full PDE formulations are provided in App.~\ref{app:pde_formulations}.

\vspace{-10pt}

\begin{figure}[t]
    \centering
\captionsetup{skip=3pt, belowskip=-24pt}
    \includegraphics[width=0.5\textwidth]{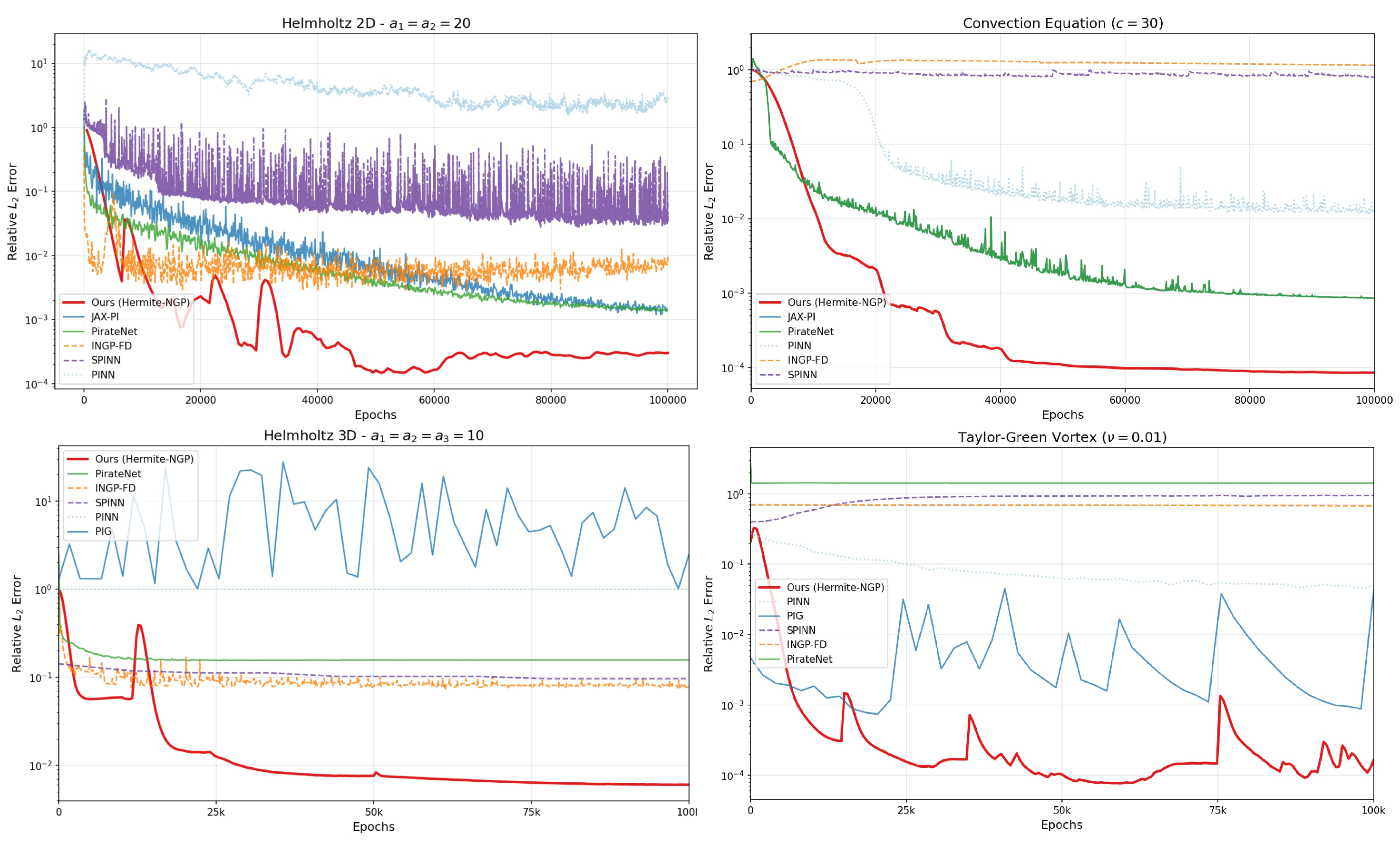}
    \caption{Training loss curves across four different benchmarks. Our method’s loss curve decreases smoothly, in contrast to the highly oscillatory behavior of the baselines, and converges to errors up to \textbf{four orders of magnitude} lower across different benchmarks; see Table~\ref{tab:main_results} for details.}
    \label{fig:loss_curves}
\end{figure}

\subsection{2D PDE Experiments}
\label{sec:2d_exp}

We evaluate two 2D benchmarks that stress high-frequency oscillations and sharp transport.
For \textbf{Helmholtz 2D} (Figs.~\ref{fig:helm_2d_10}, \ref{fig:helm_2d_20}, and \ref{fig:helm_2d_100}),
Hermite-NGP achieves relative $L^2$ errors of $1.81{\times}10^{-5}$, $7.93{\times}10^{-5}$, and
$4.59{\times}10^{-2}$ for $a{=}10,20,100$, corresponding to $20{\times}$, $17{\times}$, and $11{\times}$
improvements over the strongest baselines, respectively. Hermite-NGP is also the only method that converges
at the challenging $a{=}100$ setting. For \textbf{time-dependent Convection ($1{+}1$D)} with $c{=}30$
(Fig.~\ref{fig:convection}), Hermite-NGP attains a relative error of $8.49{\times}10^{-5}$, a $10{\times}$
reduction compared to PirateNet ($8.54{\times}10^{-4}$). Grid-based methods (INGP-FD, SPINN) fail to resolve
the sharp front due to severe numerical diffusion from their finite-difference discretization.

\paragraph{Additional Baselines (Helmholtz 2D, $a{=}10$).}
Under a unified protocol, $\partial^\infty$-Grid~\citep{kairanda2026dinf} reaches $6.07{\times}10^{-3}$ and SIREN~\citep{sitzmann2020implicit} $6.67{\times}10^{-2}$ -- both two to three orders of magnitude behind our result in Table~\ref{tab:main_results} (Fig.~\ref{fig:helm_a10_compare}; Appendix Table~\ref{tab:helm_extra_baselines}).

\begin{figure}[h]
\centering
\includegraphics[width=1.0\linewidth]{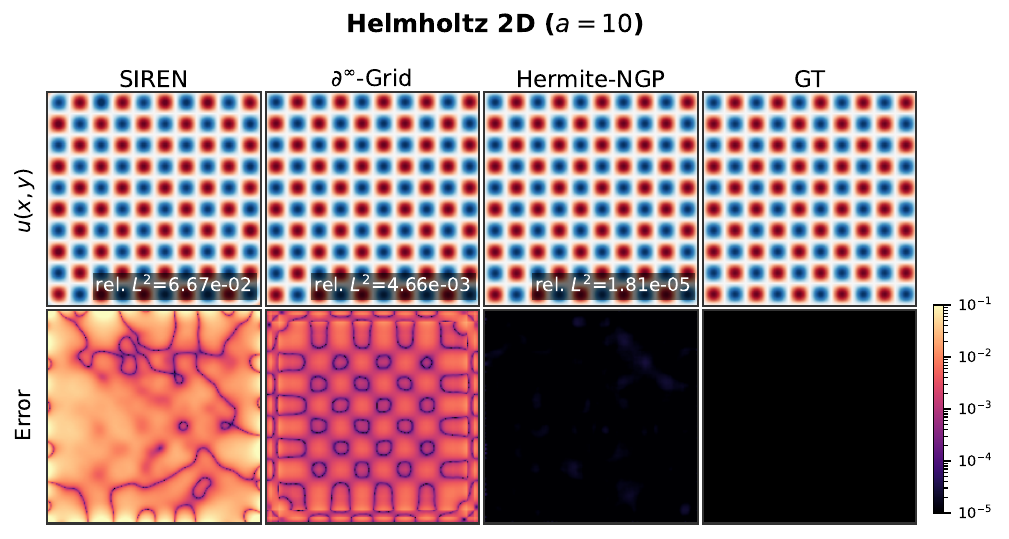}
\caption{Helmholtz 2D ($a{=}10$) qualitative comparison across Hermite-NGP, $\partial^\infty$-Grid~\citep{kairanda2026dinf}, SIREN. Only Hermite-NGP recovers the high-frequency structure cleanly.}
\label{fig:helm_a10_compare}
\end{figure}

\vspace{-10pt}

\subsection{3D PDE Experiments}
\label{sec:3d_exp}

\begin{figure}[t]
    \centering
    \includegraphics[width=1.0\linewidth]{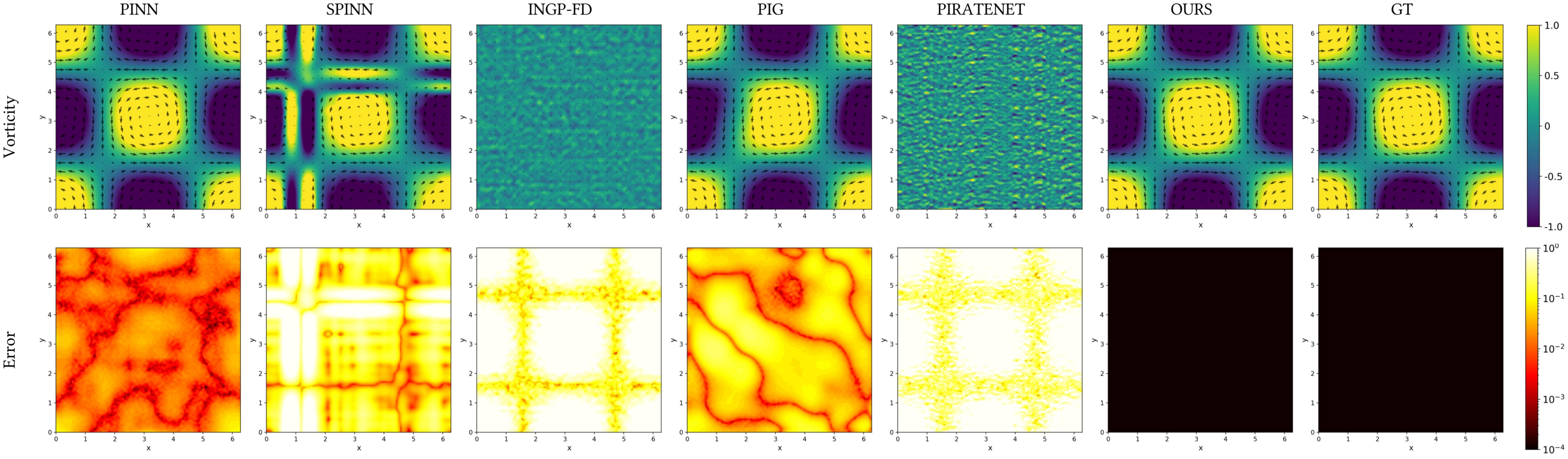}
     \caption{Taylor-Green vortex ($\nu=0.01$): Velocity magnitude at $t=1$. Hermite-NGP captures the decaying vortex
  structure with $L^2 = 7.71{\times}10^{-5}$, outperforming PIG by $9{\times}$.}
    \label{fig:taylor_green}
\end{figure}

\begin{figure}
    \centering
\captionsetup{belowskip=-16pt}
    \includegraphics[width=1.0\linewidth]{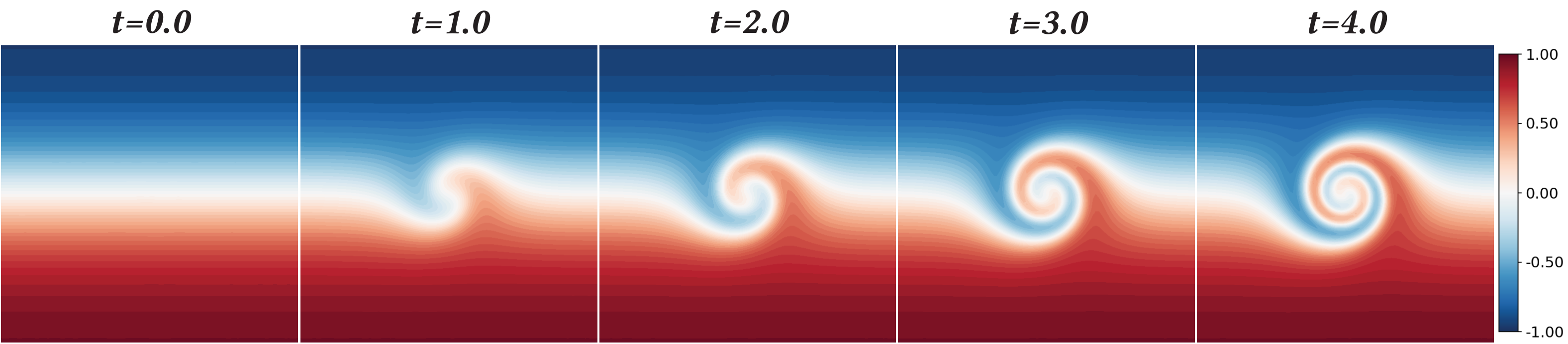}
    \caption{Flow Mixing: Solution field showing rotational transport. Hermite-NGP achieves $L^2 = 2.35{\times}10^{-4}$,
  resolving sharp gradients from fluid stretching.}
    \label{fig:flow_mixing}
\end{figure}

We evaluate three 3D benchmarks, including a stationary Helmholtz problem and two time-dependent $2{+}1$D
flows. For \textbf{Helmholtz 3D} (Figs.~\ref{fig:helm_3d_a10}, \ref{fig:helm_3d_a3}), Hermite-NGP achieves
relative $L^2$ errors of $6.09{\times}10^{-5}$ at $a{=}3$ (14$\times$ lower than PirateNet) and
$6.01{\times}10^{-3}$ at $a{=}10$ (12$\times$ lower than INGP-FD), demonstrating strong scalability to 3D.
For \textbf{time-dependent Taylor–Green vortex ($2{+}1$D)} (Fig.~\ref{fig:taylor_green}), Hermite-NGP attains
a relative $L^2$ error of $7.71{\times}10^{-5}$, which is 9$\times$ lower than PIG and over 600$\times$ lower
than the vanilla PINN baseline, while PirateNet fails to converge on this coupled system. For \textbf{time-dependent flow mixing ($2{+}1$D)} (Figs.~\ref{fig:flow_mixing} and \ref{fig:flow_compare}),
Hermite-NGP reaches a relative $L^2$ error of $2.35{\times}10^{-4}$, slightly outperforming PIG
($2.67{\times}10^{-4}$) and achieving a 12$\times$ improvement over SPINN; the PIG and SPINN numbers are taken from their original publications. For the first two  benchmarks, the second-row error visualizations show that Hermite-NGP consistently has the smallest discrepancy from the ground truth
compared to all baselines. Quantitative results are summarized in Table~\ref{tab:main_results}.

\vspace{-6pt}




\subsection{3D PDEs on Complex Geometries}
\label{sec:complex_geometry}

We evaluate Hermite-NGP on 3D Stanford meshes with varying geometric complexity, assessing the ability of neural PDE solvers to handle intricate shapes, complex boundaries, and higher-order geometric operators such as curvature.

\vspace{-10pt}

\paragraph{3D Poisson with Mesh Boundary Conditions.}
We solve a homogeneous Poisson problem (Laplace equation) $\Delta u = 0$ on $[0,1]^3$, with the mesh surface as an inner Dirichlet boundary ($u=1$) and the domain boundary as an outer Dirichlet boundary ($u=0$). Ground truth is computed using a high-resolution conjugate-gradient solver in SciPy~\citep{virtanen2020scipy}.
Table~\ref{tab:complex3d} reports relative $L^2$ errors, where ours achieves, on average, a
$3{\times}$ lower error than PIG.

\vspace{-8pt}
\begin{figure}
    \centering
\captionsetup{belowskip=-6pt}
    \includegraphics[width=1.0\linewidth]{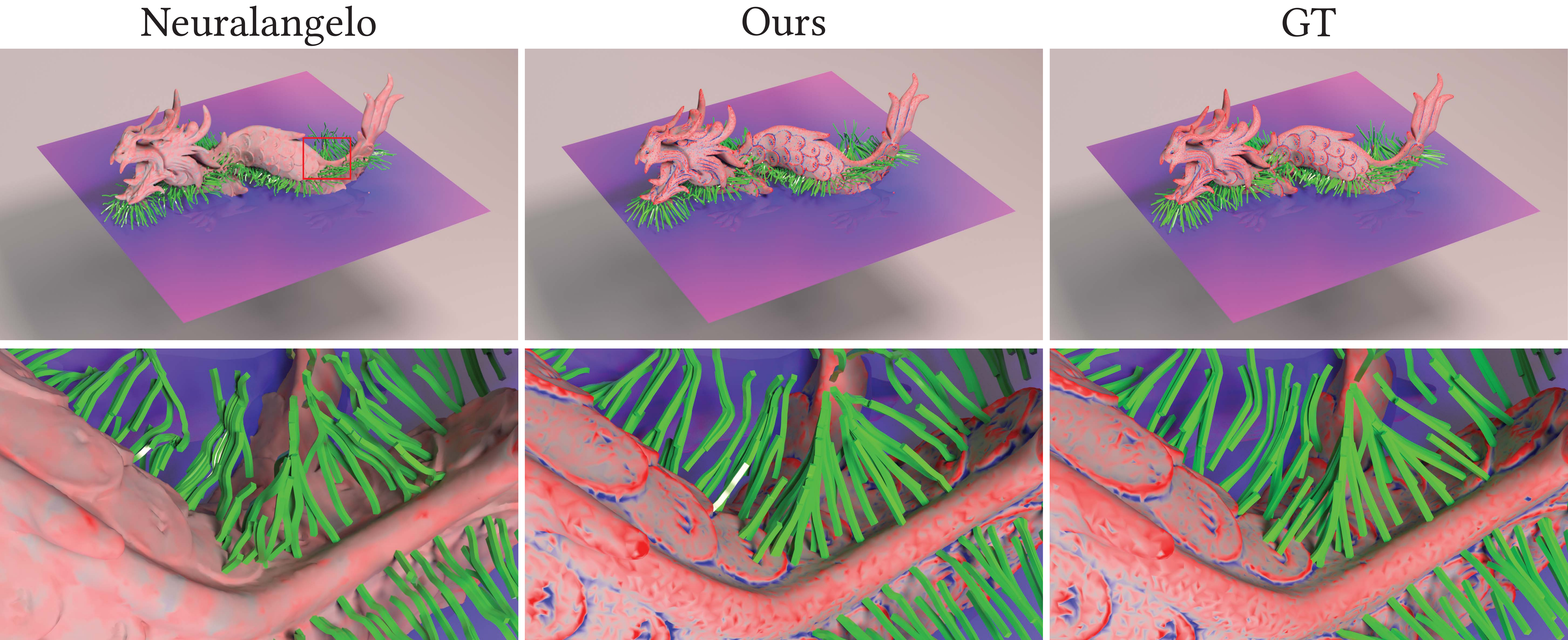}
    \caption{\textbf{SDF Learning and Curvature Estimation.} Hermite-NGP learns an SDF whose gradients (green line segments) and curvature (shown as color) remain smooth thanks to analytic second derivatives, whereas NeuralAngelo's finite-difference curvature is noisy and exhibits visible artifacts.}
    \label{fig:sdf}
\end{figure}
\paragraph{SDF learning.}
Table~\ref{tab:complex3d} reports the mean absolute error (MAE) of the  gradient, where Hermite-NGP attains a $2.4{\times}$ lower error than NeuralAngelo~\citep{li2023neuralangelo} on average. As shown in Fig.~\ref{fig:sdf}, our analytic second derivatives produce smooth and accurate curvature shown as color on mesh, whereas
NeuralAngelo’s finite-difference approximations yield noisy curvature with visible artifacts, a critical limitation for downstream tasks such as mesh extraction and physics simulation.

\begin{table}[t]
\centering
\small
\caption{\textbf{3D with Complex Geometry.} Poisson with mesh boundary (MAE $\downarrow$) and SDF learning (Grad MAE $\downarrow$).}
\label{tab:complex3d}
\begin{tabular}{l|cc|cc}
\toprule
& \multicolumn{2}{c|}{3D Poisson: L2 $\downarrow$} & \multicolumn{2}{c}{SDF: Grad. MAE $\downarrow$} \\
Mesh & Ours & PIG & Ours & NeuralAngelo \\
\midrule
Armadillo & \textbf{0.0055} & 0.0167 & \textbf{0.0478} & 0.1009 \\
Bunny     & \textbf{0.0044} & 0.0127 & \textbf{0.0416} & 0.0887 \\
Fandisk   & \textbf{0.0031} & 0.0100 & \textbf{0.0516} & 0.1064 \\
Lucy      & --              & --     & \textbf{0.0418} & 0.1213 \\
Dragon    & --              & --     & \textbf{0.0453} & 0.1322 \\
\midrule
Average   & \textbf{0.0043} & 0.0131 & \textbf{0.0456} & 0.1099 \\
\bottomrule
\end{tabular}
\end{table}

\vspace{-10pt}

\subsection{Image Reconstruction from Gradient Supervision}
\label{sec:image_recon}

To show that the gradient-augmented hash representation generalizes beyond PDE residuals, we reconstruct the \texttt{camera} image from \emph{gradient}-only supervision following~\citet{kairanda2026dinf}. Hermite-NGP attains the best PSNR at both resolutions ($32.56$\,dB at $256{\times}256$, $32.35$\,dB at $512{\times}512$), ahead of $\partial^\infty$-Grid~\citep{kairanda2026dinf} and SIREN~\citep{sitzmann2020implicit} (Fig.~\ref{fig:image_recon_compare}; best result in Appendix Table~\ref{tab:image_recon}).

\begin{figure}[h]
\centering
\includegraphics[width=1.0\linewidth]{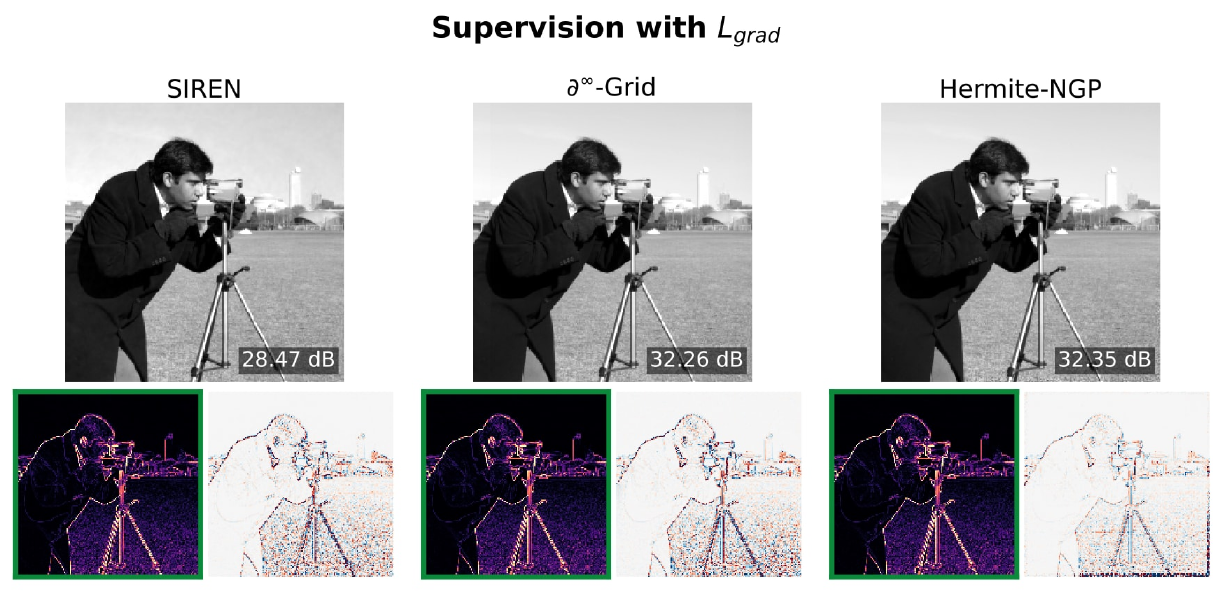}
\caption{Image reconstruction from \emph{gradient} supervision on the \texttt{camera} image. Hermite-NGP recovers sharper edges and finer texture than $\partial^\infty$-Grid and SIREN while matching their soft-region quality.}
\label{fig:image_recon_compare}
\end{figure}

\vspace{-10pt}

\subsection{GPU Memory and Speed Comparison}
Hermite-NGP scales efficiently: 68K to 16.8M parameters uses 133--389\,MB memory and 1.8--3.6\,ms/epoch
(Table~\ref{tab:memory}). INGP-FD with similar parameter counts (81K--11.2M) uses 5.4--241\,MB and 2.5--8.7\,ms/epoch.
PIG scales poorly: 1600 Gaussians (32K params) requires 33.5\,GB memory and 5\,s/epoch. Full details are provided in Appendix Table~\ref{tab:memory}.

\paragraph{3D Training Memory vs.~INGP-FD.}
Despite the $2^d$ storage overhead, in 3D Hermite-NGP uses \emph{less} peak GPU memory than INGP-FD across the tested range ($150$K--$33.6$M params), because INGP-FD's $7$ forward passes for the central-difference Laplacian retain $7\times$ activation graphs while Hermite-NGP keeps a single graph (Appendix Table~\ref{tab:mem3d}).

\vspace{-10pt}
\begin{figure}[t]
    \centering
\captionsetup{belowskip=-17pt}
    \includegraphics[width=1.0\linewidth]{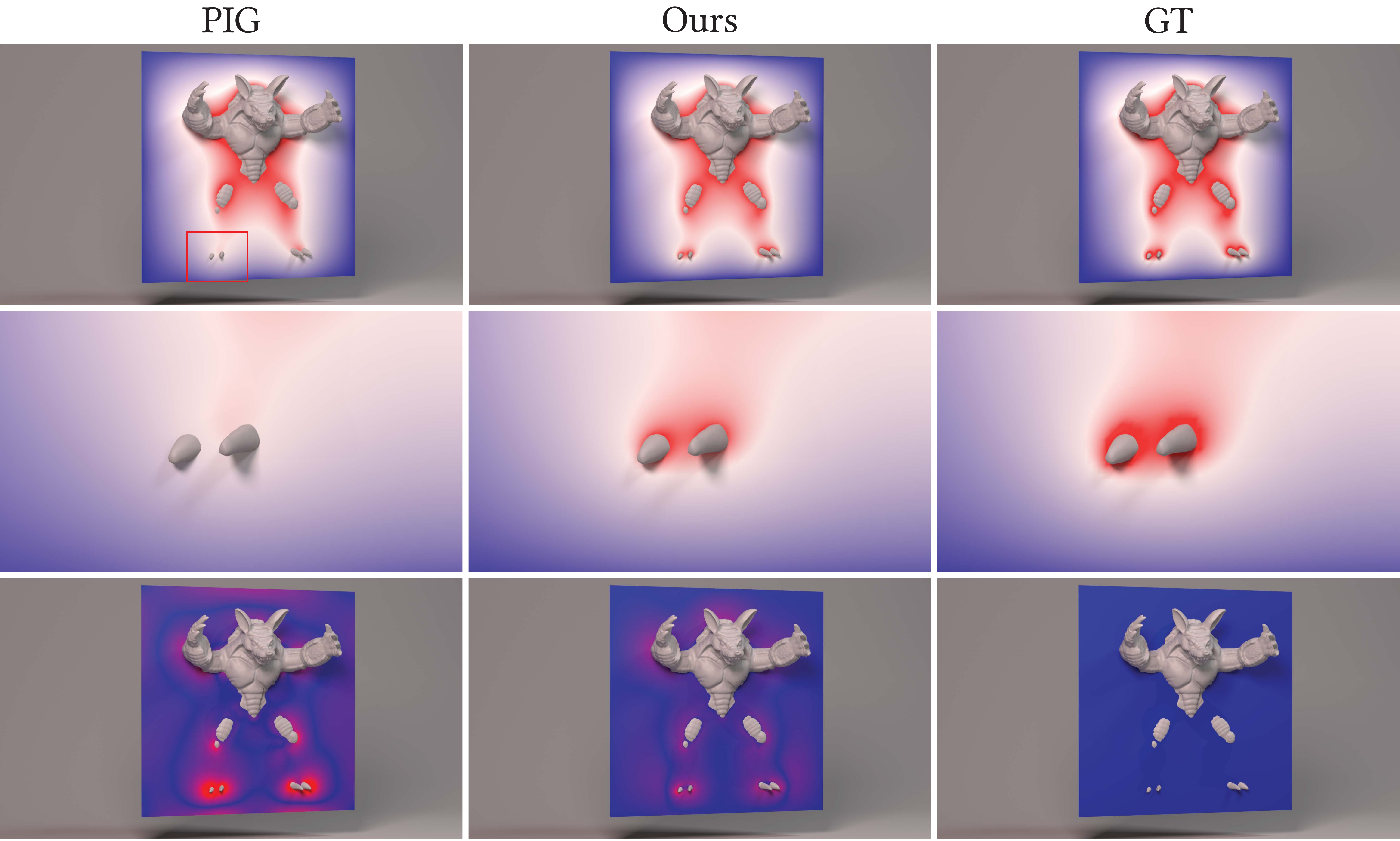}
    \caption{3D Poisson with mesh boundary (Armadillo). Solution $u$ with Dirichlet values $u{=}1$ on the mesh surface and $u{=}0$ on the outer domain boundary. Hermite-NGP achieves an $L^2$ error of $5\times10^{-3}$, a $3\times$ improvement over PIG. The third row shows L2 error.}
    \label{fig:poisson}
\end{figure}

\subsection{Ablation Studies}
\label{sec:ablation}

We conduct ablation studies on Helmholtz 2D ($a=5, 10, 15, 20$) to analyze key design choices. We report key findings here and refer to Appendix~\ref{app:ablation} for complete results.

\vspace{-10pt}

\paragraph{Cubic-NGP Baseline.}
We evaluate higher-order NGP variants (Trilinear, Cubic, Bicubic Catmull--Rom) that differentiate through the encoding without stored derivatives. On Helmholtz 2D ($a{=}10$), all variants fail (L2 $>0.1$), including true $4{\times}4$ Bicubic, while Hermite-NGP achieves $1.81\times10^{-5}$. The underlying failure mode is that hash collisions inject high-frequency noise into the stored feature values; computing derivatives through (or by FD on) these noisy features amplifies the noise. Hermite-NGP instead treats the derivative field as a first-class optimization target, distributing the representation burden across function and derivative channels. Figure~\ref{fig:ngpcubic_helm_ablation} provides a visual comparison and confirms that stored derivative coefficients are essential.

\vspace{-10pt}

\paragraph{Coarse-to-Fine Training.}

We evaluate several multigrid-inspired resolution schedules. Our coarse-to-fine (C2F) strategy reduces error by 79.2\% relative to training without any scheduling, and consistently outperforms both V- and W-cycle variants, which yield smaller improvements. Detailed results for all 28 configurations are provided in Appendix~\ref{app:c2f_ablation}.

\vspace{-10pt}

\paragraph{Hash Table Allocation.}~\label{subsec:hash-ablation}
We ablate hash capacity for Hermite coefficients: values ($H_1$), first ($H_2$), and mixed ($H_3$) derivatives. The optimal $14$–$14$–$10$ (log$_2$) achieves $2.26\times10^{-5}$ L2 error, a 56\% reduction over uniform $12$–$12$–$12$. Reducing $H_2$ to 10 degrades error to $9.98\times10^{-5}$, while reducing $H_1$ yields $8.90\times10^{-5}$, showing higher collision sensitivity for first derivatives. Appendix~\ref{app:hash_allocation} reports all 30 settings.

\vspace{-10pt}

\paragraph{Architecture Parameters.}
We ablate hash size, resolution levels, and per-level scale on 2D Helmholtz ($a{=}5,10,20$). Performance is frequency-dependent: at $a{=}5$, hash $2^{16}$ achieves $6.52\times10^{-6}$; at $a{=}10$, a compact model yields $5.29\times10^{-5}$, improved to $1.81\times10^{-5}$ with hash $2^{14}$; at $a{=}20$, a compact model yields $9.87\times10^{-5}$, improved to $7.93\times10^{-5}$ with hash $2^{16}$. Scale $2.0$ and 8 levels perform robustly across frequencies. Full sweeps are in Appendix~\ref{app:architecture_ablation}.

\vspace{-10pt}

\paragraph{MLP Depth.}
On Helmholtz 2D ($a{=}20$), depth $4$ improves error by $1.7{\times}$ over the default depth $2$ at a cost of $\sim$0.9\,ms/epoch per layer; the hash encoding carries most of the representational burden, so we keep $d{=}2$ as default (Appendix Table~\ref{tab:depth}).

\vspace{-10pt}

\paragraph{Computational Efficiency.}
Training takes 2.85--3.13 ms per epoch with 10K collocation points, broken down into 0.35--0.41 ms for PDE residuals, 0.23--0.26 ms for boundary conditions, and 1.98--2.17 ms for the backward pass. With 100K points, each epoch takes 19.7 ms, of which the backward pass accounts for 16.5 ms (84\%). See Appendix~\ref{app:speed} for a full scaling analysis across 5K--100K collocation points.

\vspace{-10pt}

\paragraph{Robustness Across Seeds.}
Across $5$ seeds on every main benchmark, relative standard deviations stay below ${\sim}15\%$ (e.g., Helmholtz 2D $a{=}10$: $3.35{\times}10^{-5}\pm4.4{\times}10^{-6}$; $a{=}100$: $7.34{\times}10^{-2}\pm6.06{\times}10^{-3}$), confirming the gains are not seed-driven (Appendix Table~\ref{tab:multiseed}).
\vspace{-10pt}

\paragraph{Derivative Computation.}
We compare full analytic derivatives to autograd-based alternatives. Our method achieves {3--15$\times$ speedup} (9.5$\times$ avg.) over full autograd. A hybrid with analytic encoding but autograd MLP is only 1.2--1.5$\times$ slower, indicating analytic encoding derivatives dominate the speedup. Gains are largest at small batches (15$\times$ at 5K collocation) due to near-constant autograd overhead. Full results are in Appendix~\ref{app:speed}.

\vspace{-10pt}

\section{Conclusion}

We presented \textbf{Hermite-NGP}, a gradient-augmented multi-resolution hash encoding for neural PDE solving. By explicitly storing function values and mixed partial derivatives at hash grid vertices and reconstructing the field via Hermite interpolation, Hermite-NGP enables analytic, stable evaluation of first- and second-order differential operators, overcoming the discontinuities of standard NGP encodings. The resulting $C^1$-continuous representation, combined with spatial adaptivity and coarse-to-fine training, yields an efficient neural PDE solver for problems on complex geometries.

\vspace{-10pt}

\paragraph{Limitations and Future Work.}
Hermite-NGP incurs storage overhead from explicit mixed derivative parameterization, which may grow in higher
dimensions, though offset by reduced training memory from analytic derivatives. Our implementation uses SIREN for
efficient second-order derivatives; extending to other activation families remains future work. We also plan to explore quintic Hermite interpolation for PDEs requiring higher-order derivatives. Our Hermite interpolation is currently realized over multi-resolution hash encodings under strong-form PDE residuals; weak-form variants are left as future work, and our method is not intended as a substitute for classical FDM/FEM solvers.

\section*{Acknowledgments}

We sincerely thank the anonymous reviewers for their valuable feedback.
Georgia Tech authors acknowledge NSF CAREER \#2420319, IIS \#2433307, OISE \#2433313, IIS \#2433322, ECCS \#2318814, and CNS \#2450401 for funding support, and the NVIDIA Academic Grant Program for hardware support.
We credit the Houdini education license for rendering.

\section*{Impact Statement}
The goal of this work is to improve the reliability and efficiency of neural PDE solvers by providing analytic and faithful spatial derivatives, particularly in settings involving complex geometries and higher-order differential operators. We do not foresee direct negative societal consequences arising from this work.

\bibliography{references}

@article{raissi2017machine,
  title={Machine learning of linear differential equations using Gaussian processes},
  author={Raissi, Maziar and Perdikaris, Paris and Karniadakis, George Em},
  journal={Journal of Computational Physics},
  volume={348},
  pages={683--693},
  year={2017},
  publisher={Elsevier}
}

@article{raissi2019physics,
  title={Physics-informed neural networks: A deep learning framework for solving forward and inverse problems involving nonlinear partial differential equations},
  author={Raissi, Maziar and Perdikaris, Paris and Karniadakis, George E},
  journal={Journal of Computational physics},
  volume={378},
  pages={686--707},
  year={2019},
  publisher={Elsevier}
}

@article{muller2022instant,
  title={Instant neural graphics primitives with a multiresolution hash encoding},
  author={M{\"u}ller, Thomas and Evans, Alex and Schied, Christoph and Keller, Alexander},
  journal={ACM transactions on graphics (TOG)},
  volume={41},
  number={4},
  pages={1--15},
  year={2022},
  publisher={ACM New York, NY, USA}
}

@article{huang2024efficient,
  title={Efficient physics-informed neural networks using hash encoding},
  author={Huang, Xinquan and Alkhalifah, Tariq},
  journal={Journal of Computational Physics},
  volume={501},
  pages={112760},
  year={2024},
  publisher={Elsevier}
}

@inproceedings{li2023neuralangelo,
  title={Neuralangelo: High-fidelity neural surface reconstruction},
  author={Li, Zhaoshuo and M{\"u}ller, Thomas and Evans, Alex and Taylor, Russell H and Unberath, Mathias and Liu, Ming-Yu and Lin, Chen-Hsuan},
  booktitle={Proceedings of the IEEE/CVF Conference on Computer Vision and Pattern Recognition},
  pages={8456--8465},
  year={2023}
}

@inproceedings{barron2023zip,
  title={Zip-nerf: Anti-aliased grid-based neural radiance fields},
  author={Barron, Jonathan T and Mildenhall, Ben and Verbin, Dor and Srinivasan, Pratul P and Hedman, Peter},
  booktitle={Proceedings of the IEEE/CVF International Conference on Computer Vision},
  pages={19697--19705},
  year={2023}
}

@article{kerbl2023gaussian,
  title={3D Gaussian splatting for real-time radiance field rendering.},
  author={Kerbl, Bernhard and Kopanas, Georgios and Leimk{\"u}hler, Thomas and Drettakis, George},
  journal={ACM Trans. Graph.},
  volume={42},
  number={4},
  pages={139--1},
  year={2023}
}

@inproceedings{chen2022tensorf,
  title={Tensorf: Tensorial radiance fields},
  author={Chen, Anpei and Xu, Zexiang and Geiger, Andreas and Yu, Jingyi and Su, Hao},
  booktitle={European conference on computer vision},
  pages={333--350},
  year={2022},
  organization={Springer}
}

@article{kairanda2026dinf,
  title={DInf-Grid: A Neural Differential Equation Solver with Differentiable Feature Grids},
  author={Kairanda, Navami and Naik, Shanthika and Habermann, Marc and Sharma, Avinash and Theobalt, Christian and Golyanik, Vladislav},
  journal={arXiv preprint arXiv:2601.10715},
  year={2026}
}

@inproceedings{tancik2023nerfstudio,
  title={Nerfstudio: A modular framework for neural radiance field development},
  author={Tancik, Matthew and Weber, Ethan and Ng, Evonne and Li, Ruilong and Yi, Brent and Wang, Terrance and Kristoffersen, Alexander and Austin, Jake and Salahi, Kamyar and Ahuja, Abhik and others},
  booktitle={ACM SIGGRAPH 2023 conference proceedings},
  pages={1--12},
  year={2023}
}

@article{kang2025pig,
  title={PIG: Physics-informed gaussians as adaptive parametric mesh representations},
  author={Kang, Namgyu and Oh, Jaemin and Hong, Youngjoon and Park, Eunbyung},
  journal={arXiv preprint arXiv:2412.05994},
  year={2024}
}

@inproceedings{kang2023pixel,
  title={Pixel: Physics-informed cell representations for fast and accurate pde solvers},
  author={Kang, Namgyu and Lee, Byeonghyeon and Hong, Youngjoon and Yun, Seok-Bae and Park, Eunbyung},
  booktitle={Proceedings of the AAAI conference on artificial intelligence},
  volume={37},
  number={7},
  pages={8186--8194},
  year={2023}
}

@article{cho2023separable,
  title={Separable physics-informed neural networks},
  author={Cho, Junwoo and Nam, Seungtae and Yang, Hyunmo and Yun, Seok-Bae and Hong, Youngjoon and Park, Eunbyung},
  journal={Advances in Neural Information Processing Systems},
  volume={36},
  pages={23761--23788},
  year={2023}
}

@article{sitzmann2020implicit,
  title={Implicit neural representations with periodic activation functions},
  author={Sitzmann, Vincent and Martel, Julien and Bergman, Alexander and Lindell, David and Wetzstein, Gordon},
  journal={Advances in neural information processing systems},
  volume={33},
  pages={7462--7473},
  year={2020}
}

@article{zhongkai2024pinnacle,
  title={Pinnacle: A comprehensive benchmark of physics-informed neural networks for solving pdes},
  author={Zhongkai, Hao and Yao, Jiachen and Su, Chang and Su, Hang and Wang, Ziao and Lu, Fanzhi and Xia, Zeyu and Zhang, Yichi and Liu, Songming and Lu, Lu and others},
  journal={Advances in Neural Information Processing Systems},
  volume={37},
  pages={76721--76774},
  year={2024}
}

@inproceedings{rahaman2019spectral,
  title={On the spectral bias of neural networks},
  author={Rahaman, Nasim and Baratin, Aristide and Arpit, Devansh and Draxler, Felix and Lin, Min and Hamprecht, Fred and Bengio, Yoshua and Courville, Aaron},
  booktitle={International conference on machine learning},
  pages={5301--5310},
  year={2019},
  organization={PMLR}
}

@article{tancik2020fourier,
  title={Fourier features let networks learn high frequency functions in low dimensional domains},
  author={Tancik, Matthew and Srinivasan, Pratul and Mildenhall, Ben and Fridovich-Keil, Sara and Raghavan, Nithin and Singhal, Utkarsh and Ramamoorthi, Ravi and Barron, Jonathan and Ng, Ren},
  journal={Advances in neural information processing systems},
  volume={33},
  pages={7537--7547},
  year={2020}
}

@article{hermite1878formule,
  title={Sur la formule d'interpolation de Lagrange},
  author={Hermite, M Ch and Borchardt, M},
  journal={Journal f{\"u}r die reine und angewandte Mathematik (Crelles Journal)},
  volume={1878},
  number={84},
  pages={70--79},
  year={1878},
  publisher={De Gruyter}
}

@article{birkhoff1968piecewise,
  title={Piecewise Hermite interpolation in one and two variables with applications to partial differential equations},
  author={Birkhoff, Garrett and Schultz, Martin H and Varga, Richard S},
  journal={Numer. math},
  volume={11},
  number={3},
  pages={232--256},
  year={1968}
}

@article{ciarlet1991general,
  title={General Lagrange and Hermite interpolation in Rn with applications to finite element methods},
  author={Ciarlet, Philippe G and Raviart, Pierre-Arnaud},
  journal={Archive for Rational Mechanics and Analysis},
  volume={46},
  number={3},
  pages={177--199},
  year={1972},
  publisher={Springer}
}

@article{mildenhall2021nerf,
  title={Nerf: Representing scenes as neural radiance fields for view synthesis},
  author={Mildenhall, Ben and Srinivasan, Pratul P and Tancik, Matthew and Barron, Jonathan T and Ramamoorthi, Ravi and Ng, Ren},
  journal={Communications of the ACM},
  volume={65},
  number={1},
  pages={99--106},
  year={2021},
  publisher={ACM New York, NY, USA}
}

@article{karniadakis2021physics,
  title={Physics-informed machine learning},
  author={Karniadakis, George Em and Kevrekidis, Ioannis G and Lu, Lu and Perdikaris, Paris and Wang, Sifan and Yang, Liu},
  journal={Nature Reviews Physics},
  volume={3},
  number={6},
  pages={422--440},
  year={2021},
  publisher={Nature Publishing Group UK London}
}

@article{costabal2024delta,
  title={$\Delta$-PINNs: Physics-informed neural networks on complex geometries},
  author={Costabal, Francisco Sahli and Pezzuto, Simone and Perdikaris, Paris},
  journal={Engineering Applications of Artificial Intelligence},
  volume={127},
  pages={107324},
  year={2024},
  publisher={Elsevier}
}

@article{hu2024tackling,
  title={Tackling the curse of dimensionality with physics-informed neural networks},
  author={Hu, Zheyuan and Shukla, Khemraj and Karniadakis, George Em and Kawaguchi, Kenji},
  journal={Neural Networks},
  volume={176},
  pages={106369},
  year={2024},
  publisher={Elsevier}
}

@article{lu2021deepxde,
  title={DeepXDE: A deep learning library for solving differential equations},
  author={Lu, Lu and Meng, Xuhui and Mao, Zhiping and Karniadakis, George Em},
  journal={SIAM review},
  volume={63},
  number={1},
  pages={208--228},
  year={2021},
  publisher={SIAM}
}

@article{wang2022and,
  title={When and why PINNs fail to train: A neural tangent kernel perspective},
  author={Wang, Sifan and Yu, Xinling and Perdikaris, Paris},
  journal={Journal of Computational Physics},
  volume={449},
  pages={110768},
  year={2022},
  publisher={Elsevier}
}

@article{krishnapriyan2021characterizing,
  title={Characterizing possible failure modes in physics-informed neural networks},
  author={Krishnapriyan, Aditi and Gholami, Amir and Zhe, Shandian and Kirby, Robert and Mahoney, Michael W},
  journal={Advances in neural information processing systems},
  volume={34},
  pages={26548--26560},
  year={2021}
}

@book{briggs2000multigrid,
  title={A multigrid tutorial},
  author={Briggs, William L and Henson, Van Emden and McCormick, Steve F},
  year={2000},
  publisher={SIAM}
}

@article{wang2021understanding,
  title={Understanding and mitigating gradient flow pathologies in physics-informed neural networks},
  author={Wang, Sifan and Teng, Yujun and Perdikaris, Paris},
  journal={SIAM Journal on Scientific Computing},
  volume={43},
  number={5},
  pages={A3055--A3081},
  year={2021},
  publisher={SIAM}
}

@article{lu2021learning,
  title={Learning nonlinear operators via DeepONet based on the universal approximation theorem of operators},
  author={Lu, Lu and Jin, Pengzhan and Pang, Guofei and Zhang, Zhongqiang and Karniadakis, George Em},
  journal={Nature machine intelligence},
  volume={3},
  number={3},
  pages={218--229},
  year={2021},
  publisher={Nature Publishing Group UK London}
}

@inproceedings{li2021fourier,
  title={Fourier Neural Operator for Parametric Partial Differential Equations},
  author={Li, Zongyi and Kovachki, Nikola and Azizzadenesheli, Kamyar and Liu, Burigede and Bhattacharya, Kaushik and Stuart, Andrew and Anandkumar, Anima},
  booktitle={International Conference on Learning Representations},
  year={2021}
}

@article{wang2024respecting,
  title={Respecting causality for training physics-informed neural networks},
  author={Wang, Sifan and Sankaran, Shyam and Perdikaris, Paris},
  journal={Computer Methods in Applied Mechanics and Engineering},
  volume={421},
  pages={116813},
  year={2024},
  publisher={Elsevier}
}

@article{mcclenny2023self,
  title={Self-adaptive physics-informed neural networks},
  author={McClenny, Levi D and Braga-Neto, Ulisses M},
  journal={Journal of Computational Physics},
  volume={474},
  pages={111722},
  year={2023},
  publisher={Elsevier}
}

@inproceedings{duancopinn,
  title={CoPINN: Cognitive Physics-Informed Neural Networks},
  author={Duan, Siyuan and Wu, Wenyuan and Hu, Peng and Ren, Zhenwen and Peng, Dezhong and Sun, Yuan},
  booktitle={Forty-second International Conference on Machine Learning}
}

@article{raissi2020hidden,
  title={Hidden fluid mechanics: Learning velocity and pressure fields from flow visualizations},
  author={Raissi, Maziar and Yazdani, Alireza and Karniadakis, George Em},
  journal={Science},
  volume={367},
  number={6481},
  pages={1026--1030},
  year={2020},
  publisher={American Association for the Advancement of Science}
}

@inproceedings{chen2018gradnorm,
  title={Gradnorm: Gradient normalization for adaptive loss balancing in deep multitask networks},
  author={Chen, Zhao and Badrinarayanan, Vijay and Lee, Chen-Yu and Rabinovich, Andrew},
  booktitle={International conference on machine learning},
  pages={794--803},
  year={2018},
  organization={PMLR}
}

@article{wang2023expert,
  title={An expert's guide to training physics-informed neural networks},
  author={Wang, Sifan and Sankaran, Shyam and Wang, Hanwen and Perdikaris, Paris},
  journal={arXiv preprint arXiv:2308.08468},
  year={2023}
}

@article{wang2024piratenets,
  title={Piratenets: Physics-informed deep learning with residual adaptive networks},
  author={Wang, Sifan and Li, Bowen and Chen, Yuhan and Perdikaris, Paris},
  journal={Journal of Machine Learning Research},
  volume={25},
  number={402},
  pages={1--51},
  year={2024}
}

@inproceedings{fridovich2022plenoxels,
  title={Plenoxels: Radiance fields without neural networks},
  author={Fridovich-Keil, Sara and Yu, Alex and Tancik, Matthew and Chen, Qinhong and Recht, Benjamin and Kanazawa, Angjoo},
  booktitle={Proceedings of the IEEE/CVF conference on computer vision and pattern recognition},
  pages={5501--5510},
  year={2022}
}

@inproceedings{fridovich2023k,
  title={K-planes: Explicit radiance fields in space, time, and appearance},
  author={Fridovich-Keil, Sara and Meanti, Giacomo and Warburg, Frederik Rahb{\ae}k and Recht, Benjamin and Kanazawa, Angjoo},
  booktitle={Proceedings of the IEEE/CVF Conference on Computer Vision and Pattern Recognition},
  pages={12479--12488},
  year={2023}
}

@inproceedings{cao2023hexplane,
  title={Hexplane: A fast representation for dynamic scenes},
  author={Cao, Ang and Johnson, Justin},
  booktitle={Proceedings of the IEEE/CVF Conference on Computer Vision and Pattern Recognition},
  pages={130--141},
  year={2023}
}

@article{kim2024neuralvdb,
  title={Neuralvdb: High-resolution sparse volume representation using hierarchical neural networks},
  author={Kim, Doyub and Lee, Minjae and Museth, Ken},
  journal={ACM Transactions on Graphics},
  volume={43},
  number={2},
  pages={1--21},
  year={2024},
  publisher={ACM New York, NY}
}

@article{chen2025neural,
  title={A Neural Particle Level Set Method for Dynamic Interface Tracking},
  author={Chen, Duowen and Zhou, Junwei and Zhu, Bo},
  journal={ACM Transactions on Graphics},
  volume={44},
  number={3},
  pages={1--21},
  year={2025},
  publisher={ACM New York, NY}
}

@inproceedings{zou2024triplane,
  title={Triplane meets gaussian splatting: Fast and generalizable single-view 3d reconstruction with transformers},
  author={Zou, Zi-Xin and Yu, Zhipeng and Guo, Yuan-Chen and Li, Yangguang and Liang, Ding and Cao, Yan-Pei and Zhang, Song-Hai},
  booktitle={Proceedings of the IEEE/CVF conference on computer vision and pattern recognition},
  pages={10324--10335},
  year={2024}
}

@article{yu2022monosdf,
  title={Monosdf: Exploring monocular geometric cues for neural implicit surface reconstruction},
  author={Yu, Zehao and Peng, Songyou and Niemeyer, Michael and Sattler, Torsten and Geiger, Andreas},
  journal={Advances in neural information processing systems},
  volume={35},
  pages={25018--25032},
  year={2022}
}

@inproceedings{chen2023mobilenerf,
  title={Mobilenerf: Exploiting the polygon rasterization pipeline for efficient neural field rendering on mobile architectures},
  author={Chen, Zhiqin and Funkhouser, Thomas and Hedman, Peter and Tagliasacchi, Andrea},
  booktitle={Proceedings of the IEEE/CVF Conference on Computer Vision and Pattern Recognition},
  pages={16569--16578},
  year={2023}
}

@inproceedings{wang2023neus2,
  title={Neus2: Fast learning of neural implicit surfaces for multi-view reconstruction},
  author={Wang, Yiming and Han, Qin and Habermann, Marc and Daniilidis, Kostas and Theobalt, Christian and Liu, Lingjie},
  booktitle={Proceedings of the IEEE/CVF International Conference on Computer Vision},
  pages={3295--3306},
  year={2023}
}

@inproceedings{hu2023tri,
  title={Tri-miprf: Tri-mip representation for efficient anti-aliasing neural radiance fields},
  author={Hu, Wenbo and Wang, Yuling and Ma, Lin and Yang, Bangbang and Gao, Lin and Liu, Xiao and Ma, Yuewen},
  booktitle={Proceedings of the IEEE/CVF International Conference on Computer Vision},
  pages={19774--19783},
  year={2023}
}

@inproceedings{chen2023neurbf,
  title={Neurbf: A neural fields representation with adaptive radial basis functions},
  author={Chen, Zhang and Li, Zhong and Song, Liangchen and Chen, Lele and Yu, Jingyi and Yuan, Junsong and Xu, Yi},
  booktitle={Proceedings of the IEEE/CVF International Conference on Computer Vision},
  pages={4182--4194},
  year={2023}
}

@inproceedings{chetan2025accurate,
  title={Accurate differential operators for hybrid neural fields},
  author={Chetan, Aditya and Yang, Guandao and Wang, Zichen and Marschner, Steve and Hariharan, Bharath},
  booktitle={Proceedings of the Computer Vision and Pattern Recognition Conference},
  pages={530--539},
  year={2025}
}

@inproceedings{wang2024neural,
  title={Neural physical simulation with multi-resolution hash grid encoding},
  author={Wang, Haoxiang and Yu, Tao and Yang, Tianwei and Qiao, Hui and Dai, Qionghai},
  booktitle={Proceedings of the AAAI Conference on Artificial Intelligence},
  volume={38},
  number={6},
  pages={5410--5418},
  year={2024}
}

@article{nave2010gradient,
  title={A gradient-augmented level set method with an optimally local, coherent advection scheme},
  author={Nave, Jean-Christophe and Rosales, Rodolfo Ruben and Seibold, Benjamin},
  journal={Journal of Computational Physics},
  volume={229},
  number={10},
  pages={3802--3827},
  year={2010},
  publisher={Elsevier}
}

@article{jiang2015affine,
  title={The affine particle-in-cell method},
  author={Jiang, Chenfanfu and Schroeder, Craig and Selle, Andrew and Teran, Joseph and Stomakhin, Alexey},
  journal={ACM Transactions on Graphics (TOG)},
  volume={34},
  number={4},
  pages={1--10},
  year={2015},
  publisher={ACM New York, NY, USA}
}

@article{deng2023fluid,
  title={Fluid simulation on neural flow maps},
  author={Deng, Yitong and Yu, Hong-Xing and Zhang, Diyang and Wu, Jiajun and Zhu, Bo},
  journal={ACM Transactions on Graphics (TOG)},
  volume={42},
  number={6},
  pages={1--21},
  year={2023},
  publisher={ACM New York, NY, USA}
}

@article{li2023garm,
  title={Garm-ls: A gradient-augmented reference-map method for level-set fluid simulation},
  author={Li, Xingqiao and Ni, Xingyu and Zhu, Bo and Wang, Bin and Chen, Baoquan},
  journal={ACM Transactions on Graphics (TOG)},
  volume={42},
  number={6},
  pages={1--20},
  year={2023},
  publisher={ACM New York, NY, USA}
}

@article{bockmann2014gradient,
  title={A gradient augmented level set method for unstructured grids},
  author={B{\o}ckmann, Arne and Vartdal, Magnus},
  journal={Journal of Computational Physics},
  volume={258},
  pages={47--72},
  year={2014},
  publisher={Elsevier}
}

@article{kolomenskiy2016adaptive,
  title={Adaptive gradient-augmented level set method with multiresolution error estimation},
  author={Kolomenskiy, Dmitry and Nave, Jean-Christophe and Schneider, Kai},
  journal={Journal of Scientific Computing},
  volume={66},
  pages={116--140},
  year={2016},
  publisher={Springer}
}

@article{anumolu2018gradient,
  title={Gradient augmented level set method for phase change simulations},
  author={Anumolu, Lakshman and Trujillo, Mario F},
  journal={Journal of Computational Physics},
  volume={353},
  pages={377--406},
  year={2018},
  publisher={Elsevier}
}

@article{kolahdouz2013semi,
  title={A semi-implicit gradient augmented level set method},
  author={Kolahdouz, Ebrahim M and Salac, David},
  journal={SIAM Journal on Scientific Computing},
  volume={35},
  number={1},
  pages={A231--A254},
  year={2013},
  publisher={SIAM}
}

@article{lee2014narrow,
  title={A narrow-band gradient-augmented level set method for multiphase incompressible flow},
  author={Lee, Curtis and Dolbow, John and Mucha, Peter J},
  journal={Journal of Computational Physics},
  volume={273},
  pages={12--37},
  year={2014},
  publisher={Elsevier}
}

@article{mercier2020characteristic,
  title={The characteristic mapping method for the linear advection of arbitrary sets},
  author={Mercier, Olivier and Yin, Xi-Yuan and Nave, Jean-Christophe},
  journal={SIAM Journal on Scientific Computing},
  volume={42},
  number={3},
  pages={A1663--A1685},
  year={2020},
  publisher={SIAM}
}

@article{zhou2024eulerian,
  title={Eulerian-lagrangian fluid simulation on particle flow maps},
  author={Zhou, Junwei and Chen, Duowen and Deng, Molin and Deng, Yitong and Sun, Yuchen and Wang, Sinan and Xiong, Shiying and Zhu, Bo},
  journal={ACM Transactions on Graphics (TOG)},
  volume={43},
  number={4},
  pages={1--20},
  year={2024},
  publisher={ACM New York, NY, USA}
}

@inproceedings{sommer2022gradient,
  title={Gradient-sdf: A semi-implicit surface representation for 3d reconstruction},
  author={Sommer, Christiane and Sang, Lu and Schubert, David and Cremers, Daniel},
  booktitle={Proceedings of the IEEE/CVF Conference on Computer Vision and Pattern Recognition},
  pages={6280--6289},
  year={2022}
}

@inproceedings{huang20242d,
  title={2d gaussian splatting for geometrically accurate radiance fields},
  author={Huang, Binbin and Yu, Zehao and Chen, Anpei and Geiger, Andreas and Gao, Shenghua},
  booktitle={ACM SIGGRAPH 2024 conference papers},
  pages={1--11},
  year={2024}
}

@article{virtanen2020scipy,
  title={SciPy 1.0: fundamental algorithms for scientific computing in Python},
  author={Virtanen, Pauli and Gommers, Ralf and Oliphant, Travis E and Haberland, Matt and Reddy, Tyler and Cournapeau, David and Burovski, Evgeni and Peterson, Pearu and Weckesser, Warren and Bright, Jonathan and others},
  journal={Nature methods},
  volume={17},
  number={3},
  pages={261--272},
  year={2020},
  publisher={Nature Publishing Group US New York}
}

@article{kingma2014adam,
  title={Adam: A method for stochastic optimization},
  author={Kingma, Diederik P},
  journal={arXiv preprint arXiv:1412.6980},
  year={2014}
}
\bibliographystyle{icml2026}

\newpage
\appendix
\onecolumn

\section{Analytic Derivative Computation}
\label{app:analytic_deriv}

Here, we provide complete derivations of the analytic derivative formulas used in Section~\ref{sec:analytic_deriv} of the main text.
Our neural graphics primitives (NGP) network architecture consists of two components: a Hermite interpolation--based multi-resolution representation and a SIREN-based MLP network.
We first present the analytic derivatives for the Hermite interpolation--based multi-resolution representation, and then derive the corresponding formulas for the SIREN network.
Finally, we combine these results to obtain the complete analytic derivatives of the full model.


\subsection{Full Derivation of Encoding Derivatives}
\label{app:encoding_derivation}

We derive the first and second derivatives of the Hermite encoding in detail, culminating in a complete 2D worked example.

\paragraph{Setup.}
Consider a single resolution level with grid spacing $\Delta x$. For a query point $\mathbf{x} = (x, y)$ in 2D, let the enclosing cell have corners at $(x_0, y_0)$, $(x_0 + \Delta x, y_0)$, $(x_0, y_0 + \Delta x)$, $(x_0 + \Delta x, y_0 + \Delta x)$. We denote corners as: $(0,0) \to$ bottom-left, $(1,0) \to$ bottom-right, $(0,1) \to$ top-left, $(1,1) \to$ top-right and define local coordinates:
\begin{equation}
s = \frac{x - x_0}{\Delta x}, \quad t = \frac{y - y_0}{\Delta x}, \quad s, t \in [0, 1].
\end{equation}

\paragraph{2D Hermite Interpolant.}
The bicubic Hermite interpolant is:
\begin{equation}
\begin{aligned}
H[f](s,t) &= \sum_{i,j \in \{0,1\}} \Big[ f_{ij}^{(00)}H^{(00)}(s-i,t-i)+ f_{ij}^{(10)}H^{(10)}(s-i,t-i)\Delta x \\
        &\qquad + f_{ij}^{(01)}H^{(01)}(s-i,t-i) \Delta x + f_{ij}^{(11)}H^{(11)}(s-i,t-i)\Delta^2 x \Big]    
\end{aligned}
\label{eq:hermite_2d_full}
\end{equation}
where $f_{ij}$, $f^{(10)}_{ij}$, $f^{(01)}_{ij}$, $f^{(11)}_{ij}$ are the function value, $x$-derivative, $y$-derivative, and mixed derivative at corner $(i,j)$, and the basis functions are tensor products:
\begin{align}
H^{(00)}(s,t) &= h^{(0)}(s) h^{(0)}(t), \\
H^{(10)}(s,t) &= h^{(1)}(s) h^{(0)}(t), \\
H^{(01)}(s,t) &= h^{(0)}(s) h^{(1)}(t), \\
H^{(11)}(s,t) &= h^{(1)}(s) h^{(1)}(t).
\end{align}

\paragraph{First Derivatives.}
Using the chain rule $\partial/\partial x = (1/\Delta x) \partial/\partial s$:
\begin{equation}
\begin{aligned}
\frac{\partial H[f]}{\partial x} &= \frac{1}{\Delta x}\sum_{i,j \in \{0,1\}} \Big[ f_{ij}^{(00)} h^{(0)}{'}(s-i) h^{(0)}(t-i) + f_{ij}^{(10)}h^{(1)}{'}(s-i) h^{(0)}(t-i)\Delta x 
\\
 &\qquad  +  f_{ij}^{(01)}h^{(0)}{'}(s-i) h^{(1)}(t-i) \Delta x + f_{ij}^{(11)} h^{(1)}{'}(s-i) h^{(1)}(t-i) \Delta^2 x \Big]    
\end{aligned}
\label{eq:first_deriv_full}
\end{equation}

Simplifying the $\Delta x$ factors:
\begin{equation}
\begin{aligned}
\frac{\partial H[f]}{\partial x} &= \sum_{i,j \in \{0,1\}} \Big[ \frac{1}{\Delta x}f_{ij}^{(00)} h^{(0)}{'}(s-i) h^{(0)}(t-i) + f_{ij}^{(10)}h^{(1)}{'}(s-i) h^{(0)}(t-i) \\
&\qquad +  f_{ij}^{(01)}h^{(0)}{'}(s-i) h^{(1)}(t-i) + f_{ij}^{(11)} h^{(1)}{'}(s-i) h^{(1)}(t-i) \Delta x \Big]    
\end{aligned}
\end{equation}

\paragraph{Second Derivatives.}
Taking another derivative:
\begin{equation}
\begin{aligned}
\frac{\partial^2 H[f]}{\partial x^2} &= \sum_{i,j \in \{0,1\}} \Big[ \frac{1}{\Delta^2 x}f_{ij}^{(00)} h^{(0)}{'}(s-i) h^{(0)}(t-i) + \frac{1}{\Delta x}f_{ij}^{(10)}h^{(1)}{'}(s-i) h^{(0)}(t-i)\\
&\qquad+ \frac{1}{\Delta x} f_{ij}^{(01)}h^{(0)}{'}(s-i) h^{(1)}(t-i) + f_{ij}^{(11)} h^{(1)}{'}(s-i) h^{(1)}(t-i)\Big]    
\end{aligned}
\label{eq:second_deriv_full}
\end{equation}

\paragraph{Complete 2D Example with All 16 Terms.}  Here, we expand Equation~\eqref{eq:hermite_2d_full}.  The full interpolant has 16 terms (4 corners $\times$ 4 coefficients each):
\begin{equation}
\begin{aligned}
&H[f](s,t) =\\ 
&f_{ij}^{(00)}h^{(0)}(s)h^{(0)}(t) + f_{ij}^{(00)}h^{(0)}(s-1)h^{(0)}(t) + f_{ij}^{(00)}h^{(0)}(s)h^{(0)}(t-1) + f_{ij}^{(00)}h^{(0)}(s-1)h^{(0)}(t-1)+\\
&f_{ij}^{(10)}h^{(1)}(s)h^{(0)}(t)\Delta x+  f_{ij}^{(10)}h^{(1)}(s-1)h^{(0)}(t)\Delta x + f_{ij}^{(10)}h^{(1)}(s)h^{(0)}(t-1)\Delta x +  f_{ij}^{(10)}h^{(1)}(s-1)h^{(0)}(t-1)\Delta x+\\
&f_{ij}^{(01)}h^{(0)}(s)h^{(1)}(t) \Delta x +  f_{ij}^{(01)}h^{(0)}(s-1)h^{(1)}(t) \Delta x +  f_{ij}^{(01)}h^{(0)}(s)h^{(1)}(t-1) \Delta x+  f_{ij}^{(01)}h^{(0)}(s-1)h^{(1)}(t-1) \Delta x+\\
&f_{ij}^{(11)}h^{(1)}(s)h^{(1)}(t)\Delta^2 x + f_{ij}^{(11)}h^{(1)}(s-1)h^{(1)}(t)\Delta^2 x + \\
&f_{ij}^{(11)}h^{(1)}(s)h^{(1)}(t-1)\Delta^2 x + f_{ij}^{(11)}h^{(1)}(s-1)h^{(1)}(t-1)\Delta^2 x +        
\label{eq:16_terms}
\end{aligned}
\end{equation}

This formula makes explicit how each of the $2^2 \times 2^2 = 16$ coefficients stored in the hash tables contributes to the interpolated value.

\paragraph{Laplacian Computation.}
The Laplacian $\nabla^2 H[f] = \partial^2 H[f]/\partial x^2 + \partial^2 H[f]/\partial y^2$ is computed by evaluating~\eqref{eq:second_deriv_full} for both directions and summing. Each term involves:
\begin{itemize}
\item Retrieving 16 coefficients from hash tables (4 corners $\times$ 4 derivative types)
\item Evaluating 4 second-derivative basis values: $ h^{(0)}{''}(s), h^{(0)}{''}(s-1), h^{(1)}{''}(s), h^{(1)}{''}(s-1)$ (and similarly for $t$)
\item Computing tensor products and weighted sums
\end{itemize}

The factored structure allows efficient implementation: compute 1D basis values once per dimension, then form outer products.

\subsection{Multi-Layer SIREN Laplacian}
\label{app:siren_laplacian}

We prove by induction that the Laplacian propagates efficiently through a $K$-layer SIREN network, reusing forward pass quantities.

\begin{proposition}[SIREN Derivative Propagation]
\label{prop:siren_deriv}
Let $u(\mathbf{x}) = W_{K+1} a^{(K)} + b_{K+1}$ be a $K$-hidden-layer SIREN network with $a^{(0)} = \gamma(\mathbf{x})$ (the Hermite encoding) and $a^{(k)} = \sin(\omega z^{(k)})$, $z^{(k)} = W_k a^{(k-1)} + b_k$ for $k = 1, \ldots, K$. Then the spatial derivatives satisfy:
\begin{align}
\frac{\partial a^{(k)}}{\partial x_i} &= \omega \cos(\omega z^{(k)}) \odot W_k \frac{\partial a^{(k-1)}}{\partial x_i}, \label{eq:siren_first_app} \\
\frac{\partial^2 a^{(k)}}{\partial x_i^2} &= -\omega^2 a^{(k)} \odot \left(W_k \frac{\partial a^{(k-1)}}{\partial x_i}\right)^{\!2} + \omega \cos(\omega z^{(k)}) \odot W_k \frac{\partial^2 a^{(k-1)}}{\partial x_i^2}, \label{eq:siren_second_app}
\end{align}
where $\odot$ denotes element-wise multiplication.
\end{proposition}

\begin{proof}
We proceed by induction on layer index $k$.

\textbf{Base case ($k=0$):} The encoding derivatives $\partial \gamma/\partial x_i$ and $\partial^2 \gamma/\partial x_i^2$ are computed analytically from the Hermite basis as derived in Appendix~\ref{app:encoding_derivation}.

\textbf{Inductive step:} Assume~\eqref{eq:siren_first_app}--\eqref{eq:siren_second_app} hold for layer $k-1$. For layer $k$:

\emph{First derivative:} By the chain rule,
\begin{align*}
\frac{\partial a^{(k)}}{\partial x_i} &= \frac{\partial}{\partial x_i} \sin(\omega z^{(k)}) \\
&= \omega \cos(\omega z^{(k)}) \odot \frac{\partial z^{(k)}}{\partial x_i} \\
&= \omega \cos(\omega z^{(k)}) \odot W_k \frac{\partial a^{(k-1)}}{\partial x_i},
\end{align*}
using $\partial z^{(k)}/\partial x_i = W_k \partial a^{(k-1)}/\partial x_i$ (since $b_k$ is constant).

\emph{Second derivative:} Differentiating again,
\begin{align*}
\frac{\partial^2 a^{(k)}}{\partial x_i^2} &= \frac{\partial}{\partial x_i} \left[ \omega \cos(\omega z^{(k)}) \odot W_k \frac{\partial a^{(k-1)}}{\partial x_i} \right] \\
&= -\omega^2 \sin(\omega z^{(k)}) \odot \frac{\partial z^{(k)}}{\partial x_i} \odot W_k \frac{\partial a^{(k-1)}}{\partial x_i} + \omega \cos(\omega z^{(k)}) \odot W_k \frac{\partial^2 a^{(k-1)}}{\partial x_i^2} \\
&= -\omega^2 a^{(k)} \odot \left(W_k \frac{\partial a^{(k-1)}}{\partial x_i}\right)^{\!2} + \omega \cos(\omega z^{(k)}) \odot W_k \frac{\partial^2 a^{(k-1)}}{\partial x_i^2}.
\end{align*}

This completes the induction. The final output derivative is:
\[
\frac{\partial^2 u}{\partial x_i^2} = W_{K+1} \frac{\partial^2 a^{(K)}}{\partial x_i^2},
\]
and the Laplacian is $\nabla^2 u = \sum_{i=1}^d \partial^2 u / \partial x_i^2$.
\end{proof}

\begin{remark}[Computational Reuse]
The key efficiency insight is that all quantities in~\eqref{eq:siren_first_app}--\eqref{eq:siren_second_app} are available from a modified forward pass:
\begin{itemize}
\item $a^{(k)} = \sin(\omega z^{(k)})$ is the standard activation
\item $\cos(\omega z^{(k)})$ can be computed alongside $\sin(\omega z^{(k)})$ at negligible extra cost
\item $W_k \partial a^{(k-1)}/\partial x_i$ uses the same weight matrices as the forward pass
\end{itemize}
No backward pass or automatic differentiation graph is needed.
\end{remark}

\subsection{Mixed Partial Derivatives}
\label{app:mixed_partials}

For PDEs involving mixed derivatives (e.g., Burgers equation, convection-diffusion), we derive the full Hessian matrix.

\paragraph{Mixed Encoding Derivative.}
The mixed second derivative of the Hermite encoding is:
\begin{equation}
\frac{\partial^2 H[f]}{\partial^2 x \partial y} = \frac{1}{\Delta^2 x} \sum_{i,j} \Big[ f^{(00)}_{ij} h^{(0)}{'}(s-i) h^{(0)}{'}(t-j) + \cdots + f^{(11)}_{ij} \Delta x \Delta y h^{(1)}{'}(s-i) h^{(1)}{'}(t-j) \Big],
\end{equation}
where we use first derivatives of basis functions in both directions.

\paragraph{Mixed Derivative Through SIREN.}
For the full network, the mixed derivative propagates as:
\begin{align}
\frac{\partial^2 a^{(k)}}{\partial x_i \partial x_j} &= -\omega^2 a^{(k)} \odot \left(W_k \frac{\partial a^{(k-1)}}{\partial x_i}\right) \odot \left(W_k \frac{\partial a^{(k-1)}}{\partial x_j}\right) + \omega \cos(\omega z^{(k)}) \odot W_k \frac{\partial^2 a^{(k-1)}}{\partial x_i \partial x_j}.
\end{align}

This allows computing the full Hessian matrix $\nabla^2 u$ analytically, which is needed for:
\begin{itemize}
\item Burgers equation: $u_t + u \cdot \nabla u = \nu \nabla^2 u$
\item Navier-Stokes: requires both Laplacian and convective derivatives
\item Geometric operators: Gaussian curvature involves determinant of Hessian
\end{itemize}

\section{Experimental Details}
\label{app:exp_details}

This appendix provides complete details for reproducing our experimental results.

\subsection{PDE Formulations}
\label{app:pde_formulations}


\subsubsection{Helmholtz 2D Equation}
The 2D Helmholtz equation models time-harmonic wave propagation:
\begin{equation}
-\Delta u(x, y) + k^2 u(x, y) = f(x, y), \quad (x, y) \in \Omega = [0, 1]^2,
\label{eq:helmholtz}
\end{equation}
with Dirichlet boundary conditions $u = g$ on $\partial\Omega$. The wave number $k = 1$ is fixed; the frequency parameter $a_1$ controls solution oscillation. We use the manufactured solution:
\begin{equation}
u^*(x, y) = \sin(a_1 \pi x) \sin(a_2 \pi y),
\end{equation}
with $a_2 = a_1$. Substituting into~\eqref{eq:helmholtz} yields the source term:
\begin{equation}
f(x, y) = \bigl(k^2 - (a_1^2 + a_2^2)\pi^2\bigr) \sin(a_1 \pi x) \sin(a_2 \pi y).
\end{equation}
Higher $a_1 \in \{10, 20, 100\}$ produces increasingly oscillatory solutions that challenge neural network spectral bias.

\subsubsection{Helmholtz 3D Equation}
The 3D extension tests scalability to higher dimensions:
\begin{equation}
-\Delta u(x, y, z) + k^2 u(x, y, z) = f(x, y, z), \quad (x, y, z) \in \Omega = [0, 1]^3,
\end{equation}
with manufactured solution:
\begin{equation}
u^*(x, y, z) = \sin(a \pi x) \sin(a \pi y) \sin(a \pi z),
\end{equation}
and corresponding source:
\begin{equation}
f(x, y, z) = \bigl(k^2 - 3a^2\pi^2\bigr) \sin(a \pi x) \sin(a \pi y) \sin(a \pi z).
\end{equation}
We test $a \in \{3, 10\}$. The $2^d = 8$ Hermite coefficients per vertex in 3D represent a larger storage overhead, making this a stress test for our approach.


\subsubsection{Convection 1+1D Equation}
The 1D convection (advection) equation describes transport phenomena:
\begin{equation}
\frac{\partial u}{\partial t} + c \frac{\partial u}{\partial x} = 0, \quad x \in [0, 2\pi], \, t \in [0, 1],
\label{eq:convection}
\end{equation}
with periodic boundary conditions $u(0, t) = u(2\pi, t)$ and initial condition $u(x, 0) = \sin(x)$. The exact solution is:
\begin{equation}
u^*(x, t) = \sin(x - ct).
\end{equation}
High velocity $c = 30$ creates sharp propagating features that require accurate first-derivative computation. For neural network training, coordinates are normalized to $[0, 1]^2$.

\subsubsection{Taylor-Green Vortex 2+1D (Navier-Stokes)}
The incompressible Navier-Stokes equations in 2D with time:
\begin{align}
\frac{\partial u}{\partial t} + u\frac{\partial u}{\partial x} + v\frac{\partial u}{\partial y} &= -\frac{\partial p}{\partial x} + \nu\left(\frac{\partial^2 u}{\partial x^2} + \frac{\partial^2 u}{\partial y^2}\right), \\
\frac{\partial v}{\partial t} + u\frac{\partial v}{\partial x} + v\frac{\partial v}{\partial y} &= -\frac{\partial p}{\partial y} + \nu\left(\frac{\partial^2 v}{\partial x^2} + \frac{\partial^2 v}{\partial y^2}\right), \\
\frac{\partial u}{\partial x} + \frac{\partial v}{\partial y} &= 0 \quad \text{(incompressibility)}.
\end{align}
Domain: $t \in [0, 1]$, $(x, y) \in [0, 2\pi]^2$. The Taylor-Green vortex provides an analytic solution:
\begin{align}
u^*(x, y, t) &= -\cos(x)\sin(y)\exp(-2\nu t), \\
v^*(x, y, t) &= \sin(x)\cos(y)\exp(-2\nu t), \\
p^*(x, y, t) &= -\tfrac{1}{4}\bigl(\cos(2x) + \cos(2y)\bigr)\exp(-4\nu t).
\end{align}
This coupled system requires accurate computation of both first derivatives (advection, pressure gradient, incompressibility) and second derivatives (viscous diffusion).

\subsubsection{Flow Mixing 2+1D}
A nonlinear advection equation with spatially-varying rotational velocity:
\begin{equation}
\frac{\partial u}{\partial t} + a(x, y)\frac{\partial u}{\partial x} + b(x, y)\frac{\partial u}{\partial y} = 0,
\end{equation}
where the velocity field describes mixing flow with angular velocity depending on radial distance:
\begin{align}
r &= \sqrt{x^2 + y^2}, \quad v_t = \frac{\tanh(r)}{\cosh^2(r)}, \\
\omega &= \frac{v_t / v_{\max}}{r}, \quad a = -\frac{v_t}{v_{\max}} \cdot \frac{y}{r}, \quad b = \frac{v_t}{v_{\max}} \cdot \frac{x}{r}.
\end{align}
Domain: $t \in [0, 4]$, $(x, y) \in [-4, 4]^2$. The exact solution:
\begin{equation}
u^*(t, x, y) = -\tanh\left(\frac{y}{2}\cos(\omega t) - \frac{x}{2}\sin(\omega t)\right).
\end{equation}
This problem develops sharp gradients as fluid parcels stretch and fold under the mixing dynamics.


For both Poisson 3D and SDF learning experiments, we evaluate on five Stanford meshes of varying geometric complexity:
\begin{itemize}
\item \textbf{Armadillo}: High-genus mesh with intricate surface details
\item \textbf{Bunny}: Classic benchmark with moderate complexity
\item \textbf{Fandisk}: CAD model with sharp edges and flat regions
\item \textbf{Lucy}: High-resolution scan with fine geometric features
\item \textbf{Dragon (xyzrgb)}: Complex topology with thin structures
\end{itemize}

\subsubsection{Poisson 3D with Mesh Boundary}
We solve the Laplace equation in a domain with an embedded mesh surface:
\begin{equation}
\Delta u = 0, \quad \mathbf{x} \in \Omega \setminus \mathcal{M},
\end{equation}
where $\Omega = [0, 1]^3$ and $\mathcal{M}$ is one of the five Stanford meshes. Boundary conditions:
\begin{align}
u &= 1 \quad \text{on } \partial\mathcal{M} \text{ (mesh surface)}, \\
u &= 0 \quad \text{on } \partial\Omega \text{ (domain boundary)}.
\end{align}
The solution is a harmonic potential field smoothly varying from 1 on the mesh surface to 0 on the domain boundary. Ground truth is computed using a high-resolution finite element solver. This benchmark tests handling of complex geometry without body-fitted meshes.

\subsubsection{SDF Learning}
We learn a signed distance function $u$ in $\Omega = [0,1]^3$ around a mesh surface $\partial\mathcal{M}$ using direct supervision against ground-truth SDF values and gradients sampled from a high-resolution reference. The loss combines a value term and a gradient term:
\begin{align}
\mathcal{L}_{\text{sdf}}  &= \|u - u_{\text{gt}}\|^2, \\
\mathcal{L}_{\text{grad}} &= \|\nabla u - \nabla u_{\text{gt}}\|^2.
\end{align}
This tests first-derivative accuracy in a purely geometric setting where the gradient field is the primary quantity of interest.

\subsection{Training Configuration}
\label{app:training_config}

Table~\ref{tab:collocation} summarizes the collocation point configurations for each benchmark.

\begin{table}[h]
\centering
\small
\caption{Collocation point configuration for each PDE benchmark.}
\label{tab:collocation}
\begin{tabular}{llccc}
\toprule
PDE & Setting & Interior Points & BC/IC Points & Epochs \\
\midrule
\multicolumn{5}{l}{\textit{Elliptic PDEs}} \\
Helmholtz 2D & $a_1 = 10$ & 10,000 & 5,000/edge & 100k \\
Helmholtz 2D & $a_1 = 20$ & 40,000 & 8,000/edge & 100k \\
Helmholtz 2D & $a_1 = 100$ & 100,000 & 25,000/edge & 100k \\
Helmholtz 3D & $a = 3$ & 40,000 & 20,000 & 100k \\
Helmholtz 3D & $a = 10$ & 40,000 & 20,000 & 100k \\
\midrule
\multicolumn{5}{l}{\textit{Time-Dependent PDEs}} \\
Convection 1+1D & $c = 30$ & 10,000 & 5,000 (IC), 5,000 (BC) & 100k \\
Taylor-Green 2+1D & $\nu = 0.01$ & 20,000 & 5,000 (IC) & 100k \\
Flow Mixing 2+1D & -- & 50,000 & 10,000 (IC), 3,000/edge & 200k \\
\midrule
\multicolumn{5}{l}{\textit{Complex Geometry (5 meshes)}} \\
Poisson 3D & mesh $\to$ 1, boundary $\to$ 0 & 50,000 & 20,000 (5,000/face) & -- \\
SDF learning & -- & 50,000 & -- & -- \\
\bottomrule
\end{tabular}
\end{table}

\paragraph{Initialization.}
Following~\citet{sitzmann2020implicit}, MLP weights are initialized from $\mathcal{U}(-\sqrt{6/n}, \sqrt{6/n})$ with
$\omega_0 = 30$ for the first layer.
Hermite hash table function values are initialized with $\mathcal{N}(0, 0.01)$; derivative coefficients are
initialized to zero, representing locally constant functions.

\paragraph{Optimizer and Learning Rate.}
All experiments use the Adam optimizer with initial learning rate $10^{-3}$ and cosine annealing schedule with warm restarts. The restart period is 10,000 iterations with restart multiplier 2 (i.e., periods of 10K, 20K, 40K, ...). Collocation points and total epochs for each benchmark are summarized in Table~\ref{tab:collocation}.

\paragraph{Loss Balancing.}
We use an adaptive gradient-based loss balancing scheme inspired by GradNorm~\citep{chen2018gradnorm}. At each iteration, we compute the gradient norms of the PDE residual loss ($g_{\text{pde}}$) and boundary condition loss ($g_{\text{bc}}$). When $g_{\text{bc}} > 10^{-8}$, the BC weight is updated via exponential moving average:
\begin{equation}
\lambda_{\text{bc}} \leftarrow 0.9 \cdot \lambda_{\text{bc}} + 0.1 \cdot \frac{g_{\text{pde}}}{g_{\text{bc}}}, \quad \lambda_{\text{bc}} \in [1, \lambda_{\max}],
\end{equation}
where $\lambda_{\max}$ is a problem-dependent cap. This balances gradient magnitudes between PDE and BC losses during training.

\paragraph{Exponential Moving Average.}
We maintain an EMA of model parameters with decay rate 0.999 for stable evaluation. Final reported errors use the EMA model weights.

\paragraph{Coarse-to-Fine Training.}
For the multi-resolution curriculum (Section 4.4), we use the Coarse-to-Fine (Equal) strategy: all $L$ resolution levels are divided into equal-duration phases, with levels activated progressively from coarse to fine. Each level is trained for $\tau = T/L$ iterations before the next level is activated, where $T$ is the total training iterations.

\paragraph{Hash Table Configuration.}
For Helmholtz 2D, we select configurations based on our ablation study (Section~\ref{sec:ablation}). For other 2D
benchmarks and 3D problems with simple geometry, we use default settings: hash table size $T = 2^{14}$, $L = 8$
resolution levels, and level ratio $r = 1.5$. For 3D problems with complex geometry (Poisson with mesh boundary, SDF
learning), we increase capacity to $T = 2^{16}$, $L = 8$ levels, and level ratio $r = 2.0$ to better capture fine
geometric details.

\paragraph{MLP Architecture.}
  For baseline methods, we use MLPs with hidden dimension 128 and 4 layers, except SPINN which uses its default
  configuration. For Hermite-NGP, we use hidden dimension 128 with 2 layers, as our hash encoding stores additional derivative coefficients; the shallower MLP balances overall memory usage while the encoding provides sufficient representational capacity. In speed benchmarks (Table~\ref{tab:memory}), all methods use identical MLP configurations (128 hidden, 2 layers) for fair comparison.

\paragraph{Hardware.}
All experiments are conducted on a single NVIDIA 4090 GPU with 24GB memory.



\subsection{Baseline Implementation Details}
\label{app:baseline_details}

All baselines use official implementations with recommended hyperparameters from the original papers:
\begin{itemize}
\item \textbf{PirateNet}~\citep{wang2024piratenets}: Official JAX implementation with 4 layers, 128 neurons.
\item \textbf{JAX-PI}~\citep{wang2023expert}: Official implementation with RAR (Residual-based Adaptive Refinement)
sampling, 4 layers, 128 neurons.
\item \textbf{PIG}~\citep{kang2025pig}: 3000 Gaussian particles with learnable positions and covariances (maximum GPU
capacity).
\item \textbf{INGP-FD}~\citep{huang2024efficient}: Same hash encoding hyperparameters as Hermite-NGP, but with
standard bilinear interpolation and finite-difference Laplacian ($\epsilon = 10^{-3}$), 4 layers, 128 neurons.
\item \textbf{SPINN}~\citep{cho2023separable}: Separable basis with default configuration (100 functions per axis, 8
layers).
\item \textbf{PIXEL}~\citep{kang2023pixel}: $256 \times 256$ learnable grid for 2D problems, 4 layers, 128 neurons.
\item \textbf{PINN}~\citep{raissi2019physics}: Standard MLP with 4 layers, 128 neurons, tanh activations.
\item \textbf{NeuralAngelo}~\citep{li2023neuralangelo}: Multi-resolution hash encoding with finite difference for
SDF learning, 2 layers, 128 neurons. We use $T = 2^{18}$ for NeuralAngelo and $T = 2^{16}$ for Hermite-NGP to match
parameter counts in SDF/curvature experiments ~\ref{sec:3d_exp}.
\item \textbf{$\partial^\infty$-Grid}~\citep{kairanda2026dinf}: Official implementation with RBF interpolation on co-located feature grids. For Helmholtz 2D ($a{=}10$) we use grid resolution $256$ with $4$ scales, learning rate $10^{-3}$, and hard boundary conditions; for image reconstruction we use the authors' default \texttt{grid\_rbf\_gradient\_image.ini} configuration at the matching resolution.
\item \textbf{SIREN}~\citep{sitzmann2020implicit}: $4$-layer, $128$-neuron MLP with sinusoidal activations and $\omega_0=30$ initialization. Used as a coordinate-network baseline for the Helmholtz $a{=}10$ and image-reconstruction experiments.
\end{itemize}

All baselines are trained with the same optimizer settings (Adam, cosine annealing) and loss balancing (GradNorm) as Hermite-NGP to ensure fair comparison.

\section{Ablation Studies}
\label{app:ablation}

We conduct comprehensive ablation studies on Helmholtz 2D to analyze the impact of key design choices.

\subsection{Coarse-to-Fine Training Ablation}
\label{app:c2f_ablation}

We compare 28 curriculum learning configurations across different strategy types and phase durations. Table~\ref{tab:c2f_full} shows results sorted by best L2 error.

\begin{table}
\centering
\caption{Curriculum learning ablation on Helmholtz 2D ($a=15$). All experiments use 100K total epochs.}
\label{tab:c2f_full}
{%
\begin{tabular}{l|r|cc|c}
\toprule
Curriculum Type & Phase Ep & Best $L_2$ & Final $L_2$ & Time (s) \\
\midrule
Coarse-to-Fine & 20K & $1.20 \times 10^{-4}$ & $1.05 \times 10^{-4}$ & 261.0 \\
Coarse-to-Fine & 40K & $1.97 \times 10^{-4}$ & $1.36 \times 10^{-4}$ & 264.1 \\
Coarse-to-Fine & 20K & $2.03 \times 10^{-4}$ & $1.40 \times 10^{-4}$ & 269.0 \\
Coarse-to-Fine & 20K & $2.22 \times 10^{-4}$ & $1.58 \times 10^{-4}$ & 266.3 \\
Coarse-to-Fine & 30K & $2.44 \times 10^{-4}$ & $1.73 \times 10^{-4}$ & 264.8 \\
Coarse-to-Fine & 20K & $2.51 \times 10^{-4}$ & $1.80 \times 10^{-4}$ & 268.4 \\
V-Cycle & 20K & $2.75 \times 10^{-4}$ & $1.94 \times 10^{-4}$ & 270.6 \\
V-Cycle & 20K & $2.75 \times 10^{-4}$ & $1.92 \times 10^{-4}$ & 265.2 \\
W-Cycle & 20K & $2.80 \times 10^{-4}$ & $1.96 \times 10^{-4}$ & 265.3 \\
V-Cycle & 10K & $2.89 \times 10^{-4}$ & $2.01 \times 10^{-4}$ & 272.2 \\
W-Cycle & 20K & $3.04 \times 10^{-4}$ & $2.09 \times 10^{-4}$ & 267.5 \\
Coarse-to-Fine & 15K & $3.14 \times 10^{-4}$ & $2.16 \times 10^{-4}$ & 266.3 \\
W-Cycle & 15K & $3.18 \times 10^{-4}$ & $2.24 \times 10^{-4}$ & 267.1 \\
Coarse-to-Fine & 10K & $3.31 \times 10^{-4}$ & $2.27 \times 10^{-4}$ & 265.5 \\
Sandwich & 20K & $3.32 \times 10^{-4}$ & $2.48 \times 10^{-4}$ & 268.9 \\
V-Cycle & 25K & $3.34 \times 10^{-4}$ & $2.57 \times 10^{-4}$ & 269.4 \\
V-Cycle & 15K & $3.72 \times 10^{-4}$ & $2.60 \times 10^{-4}$ & 269.9 \\
V-Cycle & 5K & $4.52 \times 10^{-4}$ & $3.31 \times 10^{-4}$ & 269.6 \\
W-Cycle & 25K & $4.69 \times 10^{-4}$ & $3.23 \times 10^{-4}$ & 267.1 \\
Coarse-to-Fine & 5K & $4.95 \times 10^{-4}$ & $3.40 \times 10^{-4}$ & 263.8 \\
None (Baseline) & 20K & $5.76 \times 10^{-4}$ & $4.01 \times 10^{-4}$ & 264.5 \\
Fine-to-Coarse & 20K & $6.07 \times 10^{-4}$ & $4.15 \times 10^{-4}$ & 264.6 \\
W-Cycle & 10K & $6.26 \times 10^{-4}$ & $4.33 \times 10^{-4}$ & 269.6 \\
Warmup + C2F & 20K & $6.30 \times 10^{-4}$ & $4.33 \times 10^{-4}$ & 263.6 \\
C2F Gradual & 20K & $7.14 \times 10^{-4}$ & $5.69 \times 10^{-4}$ & 265.1 \\
Oscillate & 20K & $9.31 \times 10^{-4}$ & $6.49 \times 10^{-4}$ & 267.1 \\
\bottomrule
\end{tabular}%
}
\end{table}

\paragraph{Analysis.}
Coarse-to-fine training significantly outperforms all other strategies, achieving 79\% error reduction over baseline (1.20e-4 vs 5.76e-4). V-cycle and W-cycle strategies, while better than baseline, suffer from ``forgetting'' when returning to coarse levels. Fine-to-coarse training performs poorly as high-frequency noise in early stages corrupts learning. Warmup, gradual ramping, and oscillating strategies underperform baseline, suggesting that abrupt level transitions are beneficial.

\subsection{Hash Table Allocation Ablation}
\label{app:hash_allocation}

We study how to allocate hash table capacity across the three Hermite coefficient types: values ($H_1$), first derivatives ($H_2$: $\partial u/\partial x$), and second derivatives ($H_3$: $\partial u/\partial y$). Table~\ref{tab:hash_cross} shows all 30 configurations tested on Helmholtz 2D ($a=15$), ranked by L2 error.

\begin{figure*}
    \centering
    \includegraphics[width=1.0\linewidth]{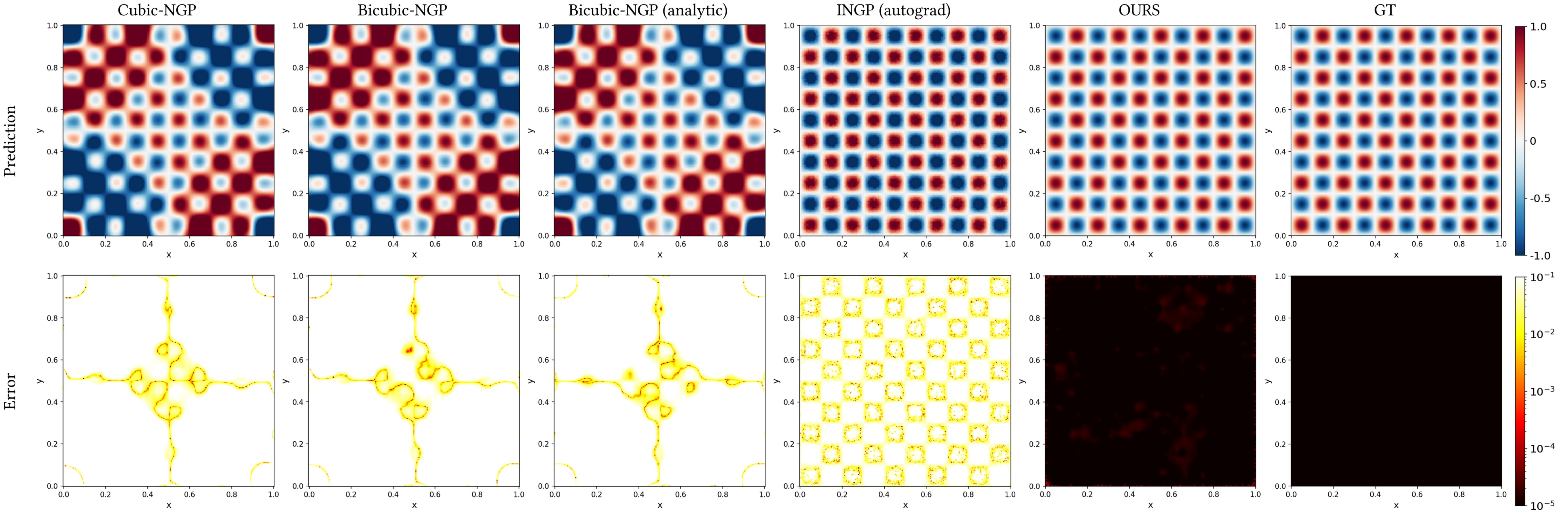}
    \caption{\textbf{Approximate Ablation.} Comparison of encoding variants on Helmholtz ($a=10$). Trilinear, Cubic, and
  Bicubic NGP all fail (L2 $>0.1$), while Hermite-NGP achieves 1.81e-5, confirming stored derivative coefficients are
  essential.}
    \label{fig:ngpcubic_helm_ablation}
\end{figure*}

\begin{table}
\centering
\caption{Complete hash table allocation ablation. $H_1$, $H_2$, $H_3$ are log$_2$ of hash table sizes for $u$, $\partial u/\partial x$, $\partial u/\partial y$ respectively.}
\label{tab:hash_cross}
{!}{%
\begin{tabular}{ccc|r|cc||ccc|r|cc}
\toprule
$H_1$ & $H_2$ & $H_3$ & Params & Best $L_2$ & Rank & $H_1$ & $H_2$ & $H_3$ & Params & Best $L_2$ & Rank \\
\midrule
14 & 14 & 10 & 822K & 2.26e-5 & 1 & 10 & 12 & 12 & 232K & 5.41e-5 & 16 \\
14 & 12 & 12 & 478K & 2.40e-5 & 2 & 14 & 10 & 14 & 576K & 5.73e-5 & 17 \\
16 & 12 & 12 & 1.26M & 2.52e-5 & 3 & 8 & 14 & 14 & 809K & 7.61e-5 & 18 \\
10 & 16 & 16 & 3.18M & 2.64e-5 & 4 & 12 & 12 & 8 & 220K & 8.27e-5 & 19 \\
12 & 14 & 12 & 674K & 2.93e-5 & 5 & 10 & 12 & 14 & 428K & 8.51e-5 & 20 \\
14 & 12 & 10 & 428K & 3.03e-5 & 6 & 16 & 12 & 8 & 1.20M & 8.78e-5 & 21 \\
12 & 16 & 10 & 2.20M & 3.47e-5 & 7 & 10 & 14 & 10 & 576K & 8.90e-5 & 22 \\
10 & 14 & 14 & 822K & 3.63e-5 & 8 & 14 & 10 & 10 & 330K & 9.98e-5 & 23 \\
8 & 16 & 16 & 3.17M & 3.66e-5 & 9 & 8 & 12 & 16 & 1.20M & 1.08e-4 & 24 \\
12 & 12 & 16 & 1.26M & 4.00e-5 & 10 & 12 & 10 & 12 & 183K & 1.13e-4 & 25 \\
12 & 12 & 10 & 232K & 4.62e-5 & 11 & 16 & 8 & 8 & 1.08M & 1.21e-4 & 26 \\
12 & 16 & 12 & 2.25M & 4.82e-5 & 12 & 8 & 12 & 12 & 220K & 1.22e-4 & 27 \\
12 & 12 & 14 & 478K & 5.11e-5 & 13 & 12 & 10 & 16 & 1.17M & 1.57e-4 & 28 \\
12 & 12 & 12 & 281K & 5.17e-5 & 14 & 10 & 10 & 14 & 330K & 1.66e-4 & 29 \\
16 & 10 & 10 & 1.12M & 5.18e-5 & 15 & 12 & 8 & 12 & 158K & 1.75e-4 & 30 \\
\bottomrule
\end{tabular}%
}
\end{table}

\paragraph{Analysis.}
Several key patterns emerge from the complete results:
\begin{itemize}
\item \textbf{$H_2$ (first derivatives) is most critical}: The top configuration uses $H_2{=}14$. Reducing $H_2$ from 14 to 10 increases error from 2.26e-5 to 9.98e-5.
\item \textbf{$H_3$ can be smaller than $H_1$, $H_2$}: The optimal $H_3{=}10$ outperforms larger $H_3$ values, suggesting the $\partial u/\partial y$ component has lower spatial variation.
\item \textbf{$H_1$ (values) has moderate sensitivity}: Reducing $H_1$ from 14 to 10 causes 294\% degradation (8.90e-5 vs 2.26e-5).
\item \textbf{Non-uniform allocation is optimal}: The best $(14, 14, 10)$ uses only 822K parameters, outperforming the larger uniform $(16, 16, 16)$ configuration.
\end{itemize}

\subsection{Architecture Parameter Ablation}
\label{app:architecture_ablation}

We ablate hash table size, number of resolution levels, and per-level scale factor across Helmholtz 2D at different frequencies ($a=5, 10, 20$). This reveals how optimal architecture depends on PDE frequency content.

\begin{table}
\centering

\caption{Architecture ablation on Helmholtz 2D ($a=5$). Top 10 configurations by L2 error.}
\label{tab:helmholtz_a5}
\begin{tabular}{ccc|r|cc}
\toprule
Hash & Scale & Lvls & Params & Best $L_2$ & Final $L_2$ \\
\midrule
16 & 2.0 & 8 & 4.21M & 6.52e-6 & 6.46e-6  \\
14 & 2.0 & 8 & 1.07M & 8.34e-6 & 8.23e-6  \\
12 & 2.0 & 8 & 281K & 8.70e-6 & 8.52e-6 \\
10 & 1.5 & 10 & 101K & 9.65e-6 & 8.63e-6  \\
14 & 1.5 & 8 & 1.07M & 1.23e-5 & 1.19e-5  \\
14 & 1.3 & 8 & 1.07M & 1.72e-5 & 1.63e-5 \\
14 & 1.8 & 8 & 1.07M & 1.73e-5 & 1.68e-5  \\
10 & 1.5 & 8 & 84K & 1.95e-5 & 1.78e-5  \\
14 & 2.0 & 6 & 805K & 1.98e-5 & 1.95e-5  \\
10 & 2.0 & 8 & 84K & 2.05e-5 & 1.86e-5 \\
\bottomrule
\end{tabular}
\end{table}

\begin{table}
\centering

\caption{Architecture ablation on Helmholtz 2D ($a=10$). Top 10 configurations by L2 error.}
\label{tab:helmholtz_a10}
\begin{tabular}{ccc|r|cc}
\toprule
Hash & Scale & Lvls & Params & Best $L_2$ & Final $L_2$  \\
\midrule
14 & 2.0 & 8 & 1.07M & 1.81e-5 & 1.66e-5   \\
16 & 2.0 & 8 & 4.21M & 3.84e-5 & 3.83e-5   \\
12 & 2.0 & 8 & 281K & 3.90e-5 & 3.37e-5  \\
14 & 2.0 & 10 & 1.33M & 4.63e-5 & 3.81e-5  \\
14 & 2.0 & 6 & 805K & 5.29e-5 & 5.08e-5  \\
14 & 1.8 & 8 & 1.07M & 5.93e-5 & 5.39e-5  \\
14 & 1.5 & 8 & 1.07M & 6.40e-5 & 5.95e-5 \\
14 & 1.8 & 6 & 805K & 9.72e-5 & 8.46e-5  \\
10 & 1.5 & 8 & 84K & 1.07e-4 & 9.10e-5 \\
10 & 1.5 & 10 & 101K & 1.33e-4 & 1.08e-4 \\
\bottomrule
\end{tabular}
\end{table}

\subsubsection{Helmholtz 2D ($a=20$) -- High Frequency}

\begin{table}
\centering

\caption{Architecture ablation on Helmholtz 2D ($a=20$). Top 10 configurations by L2 error.}
\label{tab:helmholtz_a20}
\begin{tabular}{ccc|r|cc}
\toprule
Hash & Scale & Lvls & Params & Best $L_2$ & Final $L_2$  \\
\midrule
16 & 2.0 & 8 & 4.21M & 7.93e-5 & 6.97e-5  \\
14 & 1.5 & 8 & 1.07M & 3.33e-4 & 2.65e-4  \\
14 & 2.0 & 6 & 805K & 3.37e-4 & 2.89e-4  \\
14 & 2.3 & 8 & 1.07M & 5.00e-4 & 4.40e-4  \\
12 & 2.0 & 8 & 281K & 5.53e-4 & 4.60e-4  \\
14 & 1.5 & 6 & 805K & 5.68e-4 & 4.57e-4  \\
14 & 2.0 & 8 & 1.07M & 6.38e-4 & 6.32e-4  \\
14 & 2.0 & 10 & 1.33M & 6.56e-4 & 5.42e-4  \\
14 & 1.8 & 8 & 1.07M & 6.64e-4 & 5.47e-4  \\
14 & 1.8 & 6 & 805K & 1.37e-3 & 1.21e-3  \\
\bottomrule
\end{tabular}
\end{table}

\begin{table}
\centering
\caption{Optimal configuration comparison across frequencies.}
\label{tab:helmholtz_comparison}
\begin{tabular}{c|ccc|r|cc}
\toprule
$a$ & Hash & Scale & Lvls & Params & Best $L_2$  \\
\midrule
5 & 16 & 2.0 & 8 & 4.21M & 6.52e-6 \\
10 & 14 & 2.0 & 8 & 1.07M & 1.81e-5  \\
20 & 16 & 2.0 & 8 & 4.21M & 7.93e-5 \\
\bottomrule
\end{tabular}
\end{table}

\paragraph{Analysis.}
\begin{itemize}
\item \textbf{Hash table size}: Higher frequencies ($a=20$) benefit more from larger hash tables ($2^{16}$), while medium frequencies ($a=10$) work well with $2^{14}$. This suggests high-frequency solutions have finer spatial structure requiring more hash capacity.
\item \textbf{Scale and levels}: Scale 2.0 with 8 levels consistently performs well across all frequencies. This combination provides good frequency coverage from coarse to fine scales.
\item \textbf{Speed-accuracy tradeoff}: At $a=10$, the optimal $2^{14}$ configuration achieves comparable accuracy to $2^{16}$ while being 4$\times$ smaller and faster (3.08 vs 3.40 ms/ep).
\end{itemize}

\subsection{Computational Scaling}
\label{app:speed}

We benchmark training speed across varying numbers of collocation points and hash sizes, with all timings measured on a single NVIDIA RTX 4090 GPU. For Table~\ref{tab:memory}, GradNorm is not applied, reducing backward-pass time; all other speed benchmarks include GradNorm.

\begin{table}
\centering
\caption{GPU memory comparison at 10K collocation points. Our method does not apply GradNorm, reducing backward-pass memory and computation. In contrast, I-NGP-FD relies on tiny-cuda-nn, which constrains the hash-to-level ratio and requires more levels when using larger hash tables, increasing memory usage.}

\label{tab:memory}
\begin{tabular}{l|r|c|c}
\toprule
Method & Params & Peak GPU (MB) & ms/epoch \\
\midrule
Hermite-NGP ($2^{10}$) & 68K & 133 & 1.82 \\
Hermite-NGP ($2^{12}$) & 264K & 133 & 1.99 \\
Hermite-NGP ($2^{14}$) & 1.05M & 154 & 1.86 \\
Hermite-NGP ($2^{16}$) & 4.20M & 195 & 2.01 \\
Hermite-NGP ($2^{18}$) & 16.8M & 389 & 3.62 \\
\midrule
INGP-FD ($2^{12}$, L=8) & 81K & 5.4 & 2.52 \\
INGP-FD ($2^{14}$, L=8) & 130K & 6.3 & 2.57 \\
INGP-FD ($2^{16}$, L=10) & 492K & 14.0 & 5.03 \\
INGP-FD ($2^{18}$, L=12) & 2.33M & 52.5 & 6.37 \\
INGP-FD ($2^{20}$, L=14) & 11.2M & 241.1 & 8.66 \\
\midrule
PIG (200 Gaussians) & 4K & 4,230 & 47.0 \\
PIG (400 Gaussians) & 8K & 8,390 & 96.1 \\
PIG (800 Gaussians) & 16K & 16,760 & 189.6 \\
PIG (1600 Gaussians) & 32K & 33,480 & 4994.3 \\
\bottomrule
\end{tabular}
\end{table}

\begin{table}
\centering
\caption{Computational scaling: time per epoch (ms) and component breakdown. Hash=16, 8 levels.}
\label{tab:speed}
\begin{tabular}{r|ccc|c|c}
\toprule
Collocation & PDE (ms) & BC (ms) & Backward (ms) & Total & Backward \% \\
\midrule
5K & 0.33 & 0.23 & 1.83 & 2.70 & 68\% \\
10K & 0.38 & 0.26 & 2.17 & 3.13 & 69\% \\
15K & 0.50 & 0.29 & 2.48 & 3.60 & 69\% \\
20K & 0.60 & 0.26 & 3.11 & 4.28 & 73\% \\
25K & 0.78 & 0.27 & 3.80 & 5.15 & 74\% \\
50K & 1.35 & 0.25 & 7.30 & 9.19 & 79\% \\
100K & 2.59 & 0.28 & 16.48 & 19.72 & 84\% \\
\bottomrule
\end{tabular}
\end{table}

\paragraph{Effect of Hash Table Size on Speed.}
Table~\ref{tab:speed_hash} shows timing at 10K collocation points for different hash sizes.

\begin{table}
\centering
\caption{Speed vs.\ hash table size at 20K collocation points.}
\label{tab:speed_hash}
\begin{tabular}{c|r|ccc|c}
\toprule
Hash & Params & PDE (ms) & BC (ms) & Backward (ms) & Total \\
\midrule
8 & 35K & 0.37 & 0.23 & 2.02 & 2.89 \\
10 & 84K & 0.40 & 0.26 & 2.47 & 3.42 \\
12 & 281K & 0.41 & 0.25 & 2.16 & 3.12 \\
14 & 1.07M & 0.38 & 0.26 & 2.17 & 3.13 \\
16 & 4.21M & 0.36 & 0.24 & 2.09 & 3.04 \\
18 & 16.8M & 0.38 & 0.23 & 2.40 & 4.42 \\
\bottomrule
\end{tabular}
\end{table}

\paragraph{Analysis.}
The backward pass dominates computation time, accounting for 68\% at 5K points and increasing to 84\% at 100K points. This is expected as gradient computation scales with batch size. The forward pass (PDE + BC) scales sub-linearly due to efficient hash table lookups. Hash table size has minimal impact on speed for sizes $2^8$ to $2^{16}$, with timing varying only from 2.89 to 3.04 ms/epoch at 10K points. Only at $2^{18}$ (16.8M params) does memory bandwidth become a bottleneck.

\paragraph{Derivative Computation Ablation.}
We compare three derivative computation strategies to validate our analytic derivative approach:
\begin{itemize}
\item \textbf{Ours (Full Analytic)}: CUDA Hermite encoding with analytic derivatives + CUDA MLP with analytic Laplacian
\item \textbf{Autograd-MLP}: CUDA Hermite encoding with analytic derivatives + PyTorch autograd for MLP Laplacian
\item \textbf{Fully-Autograd}: Standard NGP encoding + PyTorch autograd for all derivatives
\end{itemize}

\begin{table}
\centering
\caption{Derivative computation method comparison (hash=14, levels=8, hidden=128, layers=2). Use GradNorm}
\label{tab:deriv_ablation}
\begin{tabular}{r|ccc|c}
\toprule
Collocation & Ours (ms) & Autograd-MLP (ms) & Fully-Autograd (ms) & Speedup \\
\midrule
5K & 2.70 & 3.78 & 40.83 & 15.1$\times$ \\
10K & 3.13 & 4.68 & 42.07 & 13.4$\times$ \\
20K & 4.28 & 5.15 & 44.01 & 10.3$\times$ \\
50K & 9.19 & 10.70 & 48.77 & 5.3$\times$ \\
100K & 19.72 & 23.16 & 64.50 & 3.3$\times$ \\
\bottomrule
\end{tabular}
\end{table}

Our full analytic implementation achieves \textbf{3--15$\times$ speedup} over fully-autograd baselines (average 9.5$\times$). The Autograd-MLP hybrid is only 1.2--1.5$\times$ slower than our approach, demonstrating that \textbf{analytic encoding derivatives provide the majority of the speedup}. Speedup is highest at smaller batch sizes (15$\times$ at 5K) because autograd overhead is relatively constant while computation scales linearly.

\begin{table}[h]
\centering
\small
\caption{Interpolation ablation on Helmholtz 2D ($a=10$). Cubic uses smoothstep, Bicubic uses Catmull-Rom. All NGP variants fail.}
\label{tab:cubic_ablation}
\begin{tabular}{lcc}
\toprule
Method & Best $L_2$ & ms/epoch \\
\midrule
Hermite-NGP (ours) & \textbf{1.81e-5} & 3.08 \\
\midrule
Trilinear-NGP & 1.46e-1 & -- \\
Cubic-NGP & 2.78e-1 & 693.6 \\
Bicubic-NGP & 2.91e-1 & 683.3 \\
Bicubic-NGP (analytic) & 2.47e-1 & 470.1 \\
\bottomrule
\end{tabular}
\end{table}


\section{Additional Experiments}
\label{app:additional_experiments}

This appendix reports the supporting numbers for experiments summarized in the main text: multi-seed robustness (Table~\ref{tab:multiseed}), MLP depth ablation (Table~\ref{tab:depth}), 3D training-memory comparison vs.~INGP-FD (Table~\ref{tab:mem3d}), additional baselines on Helmholtz 2D ($a{=}10$) (Table~\ref{tab:helm_extra_baselines}), and image reconstruction (Table~\ref{tab:image_recon}). For image reconstruction, SIREN and $\partial^\infty$-Grid use Kairanda's percentile-rescale at evaluation (their convention); Hermite-NGP uses raw clip-to-$[0,1]$ since its Dirichlet anchor fixes the integration constant.

\begin{table}[h]
\centering
\small
\caption{Multi-seed ($5$ seeds) relative $L^2$ error mean $\pm$ std on the main benchmarks, using the default configuration (hash $2^{14}$, $L{=}8$, scale $2.0$, width $128$, depth $2$). Relative standard deviations stay below $\sim$15\% across all settings.}
\label{tab:multiseed}
\begin{tabular}{lc}
\toprule
Benchmark & Relative $L^2$ (mean $\pm$ std) \\
\midrule
Helmholtz 2D ($a{=}10$)  & $3.35{\times}10^{-5}\,\pm\,4.4{\times}10^{-6}$ \\
Helmholtz 2D ($a{=}20$)  & $1.07{\times}10^{-4}\,\pm\,3.64{\times}10^{-5}$ \\
Helmholtz 2D ($a{=}100$) & $7.34{\times}10^{-2}\,\pm\,6.06{\times}10^{-3}$ \\
Helmholtz 3D ($a{=}10$)  & $7.29{\times}10^{-3}\,\pm\,1.10{\times}10^{-3}$ \\
Taylor--Green ($\nu{=}0.01$) & $5.69{\times}10^{-5}\,\pm\,3.7{\times}10^{-6}$ \\
Convection ($c{=}30$)        & $9.47{\times}10^{-5}\,\pm\,2.5{\times}10^{-5}$ \\
\bottomrule
\end{tabular}
\end{table}

\begin{table}[h]
\centering
\small
\caption{MLP depth ablation on Helmholtz 2D ($a{=}20$, width $128$, hash $2^{14}$). $^*$ paper default.}
\label{tab:depth}
\begin{tabular}{c|cc}
\toprule
Depth & Best $L^2$ & ms/epoch \\
\midrule
$1$       & $1.02{\times}10^{-3}$ & $2.17$ \\
$2^{*}$   & $6.81{\times}10^{-4}$ & $3.07$ \\
$3$       & $4.27{\times}10^{-4}$ & $3.95$ \\
$4$       & $4.03{\times}10^{-4}$ & $4.86$ \\
\bottomrule
\end{tabular}
\end{table}

\begin{table}[h]
\centering
\small
\caption{Peak GPU training memory in 3D (Helmholtz 3D, $a{=}10$, $40$k collocation points). Hermite-NGP uses less peak memory than INGP-FD across the full range because INGP-FD requires $7$ forward passes (storing $7$ activation graphs) while Hermite-NGP keeps a single graph.}
\label{tab:mem3d}
\begin{tabular}{rr|rr}
\toprule
Hermite Params & Hermite Mem & INGP-FD Params & INGP-FD Mem \\
\midrule
$150$K  & $957$\,MB  & $117$K  & $1402$\,MB \\
$2.12$M & $988$\,MB  & $314$K  & $1405$\,MB \\
$8.41$M & $1084$\,MB & $4.77$M & $1502$\,MB \\
$33.6$M & $1465$\,MB & $21.0$M & $1783$\,MB \\
\bottomrule
\end{tabular}
\end{table}

\begin{table}[h]
\centering
\small
\caption{Additional baselines on Helmholtz 2D ($a{=}10$), all run under our unified protocol (Adam, cosine LR, identical collocation budget). Companion to Figure~\ref{fig:helm_a10_compare} in the main text.}
\label{tab:helm_extra_baselines}
\begin{tabular}{lc}
\toprule
Method & Relative $L^2$ \\
\midrule
Hermite-NGP (ours) & \textbf{$3.35\!\times\!10^{-5}$} \\
INGP-FD~\citep{huang2024efficient} & $3.66\!\times\!10^{-3}$ \\
$\partial^\infty$-Grid~\citep{kairanda2026dinf} & $6.07\!\times\!10^{-3}$ \\
SIREN~\citep{sitzmann2020implicit} & $6.67\!\times\!10^{-2}$ \\
Cubic I-NGP & $2.47\!\times\!10^{-1}$ \\
\bottomrule
\end{tabular}
\end{table}

\begin{table}[h]
\centering
\small
\caption{Image reconstruction from \emph{gradient} supervision on \texttt{camera}: best PSNR (dB) at two resolutions. Companion to Figure~\ref{fig:image_recon_compare} in the main text.}
\label{tab:image_recon}
\begin{tabular}{l|c|c}
\toprule
Method & $256{\times}256$ & $512{\times}512$ \\
\midrule
Hermite-NGP (ours)                              & \textbf{32.56} & \textbf{32.35} \\
$\partial^\infty$-Grid~\citep{kairanda2026dinf} & $32.47$        & $32.27$ \\
SIREN~\citep{sitzmann2020implicit}              & $32.47$        & $29.25$ \\
\bottomrule
\end{tabular}
\end{table}

\section{Additional Figures}
\label{app:additional_figures}


\begin{figure}[H]
    \centering
    \includegraphics[width=1.0\linewidth]{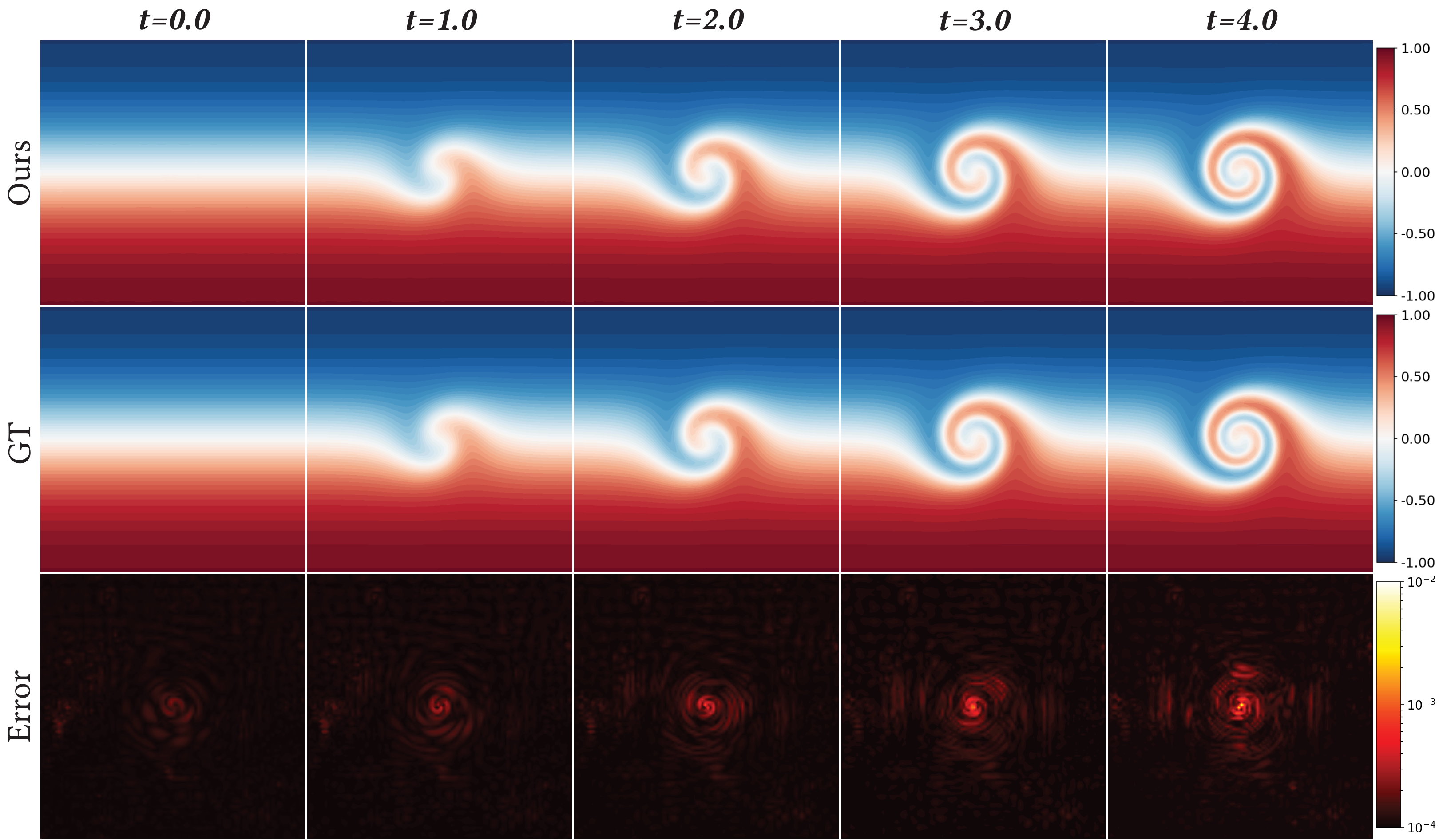}
    \caption{\textbf{Flow Mixing.} Solution snapshots at $t = 0, 0.2, 0.4, 0.6, 0.8, 1.0$. Hermite-NGP achieves average L2
   error of 2.35e-4 across all time steps.}
    \label{fig:flow_compare}
\end{figure}

\begin{figure*}
    \centering
    \includegraphics[width=1.0\linewidth]{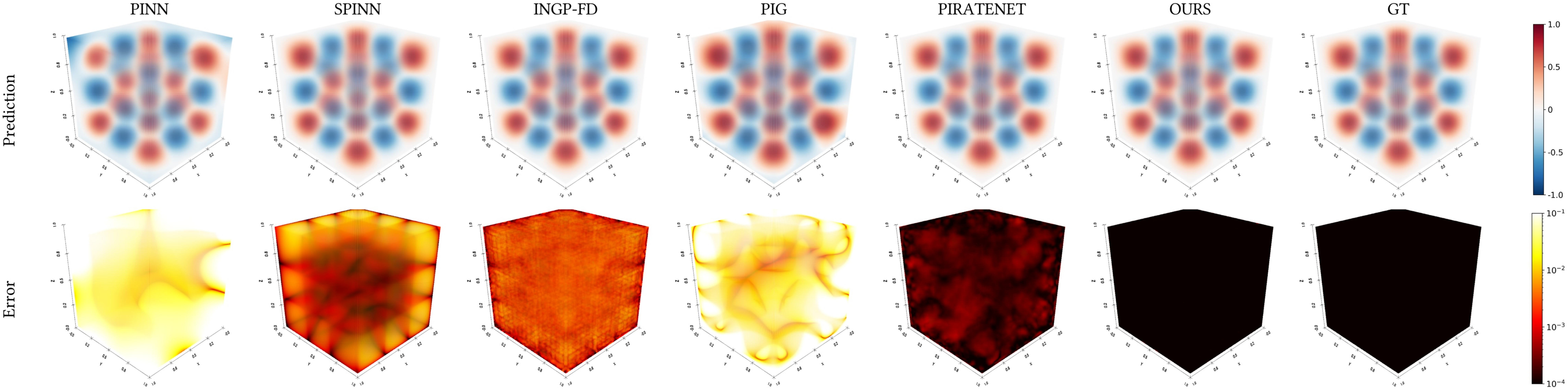}
    \caption{\textbf{Helmholtz 3D ($a=3$).} Cross-sectional slices of the 3D Helmholtz solution. Hermite-NGP achieves L2
  error of 6.09e-5, compared to 8.40e-4 for PirateNet (13.8$\times$ lower).}
    \label{fig:helm_3d_a3}
\end{figure*}

\begin{figure*}
    \centering
    \includegraphics[width=1.0\linewidth]{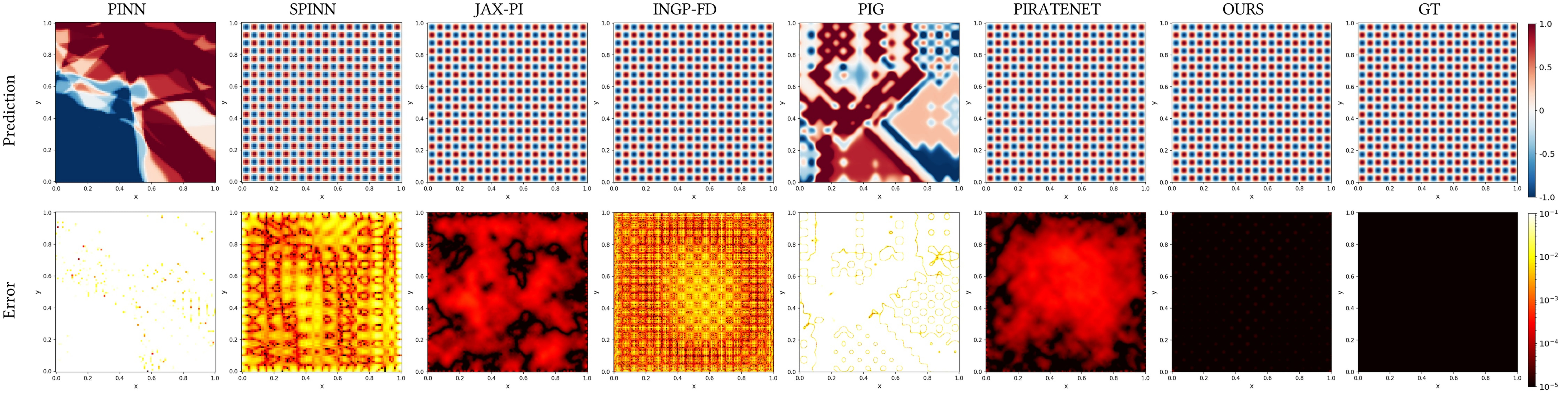}
    \caption{\textbf{Helmholtz 2D ($a=20$).} High-frequency solution with 20 wavelengths. Hermite-NGP achieves 9.87e-5
  (1.07M params) and 7.93e-5 (4.21M params), outperforming PirateNet (1.36e-3), JAX-PI (1.2e-3), INGP-FD (2.77e-3), and
  SPINN (3.30e-2).}
    \label{fig:helm_2d_20}
\end{figure*}

\begin{figure*}
    \centering
    \includegraphics[width=1.0\linewidth]{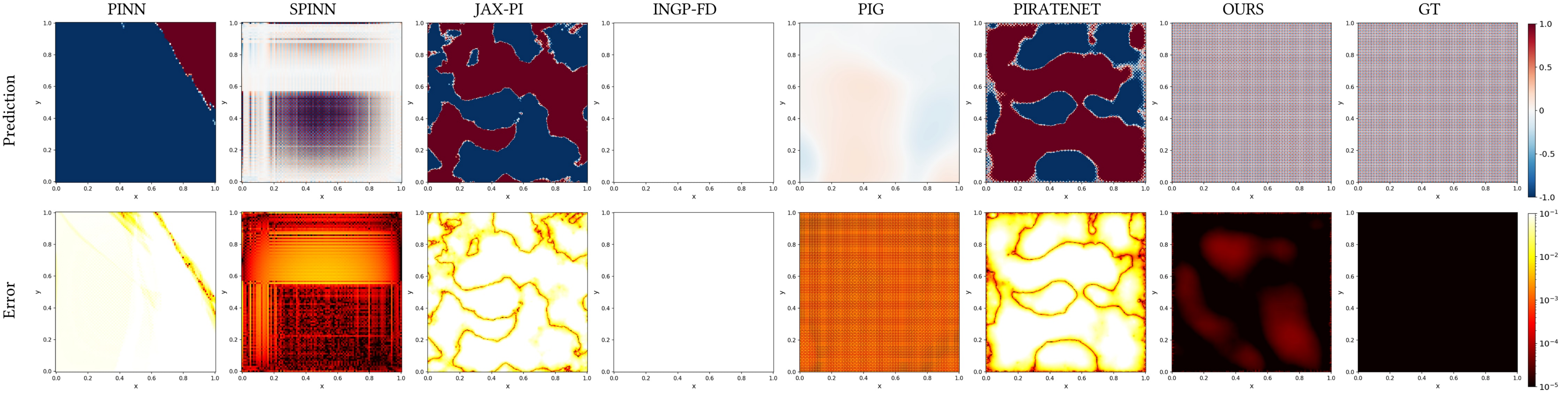}
    \caption{Helmholtz 2D ($a=100$): Hermite-NGP is the only method that converges on this challenging
  high-frequency setting, achieving $L^2 = 4.59{\times}10^{-2}$. All baseline methods fail to capture the rapid
  oscillations.}
    \label{fig:helm_2d_100}
\end{figure*}

\begin{figure}
    \centering
    \includegraphics[width=1.0\textwidth]{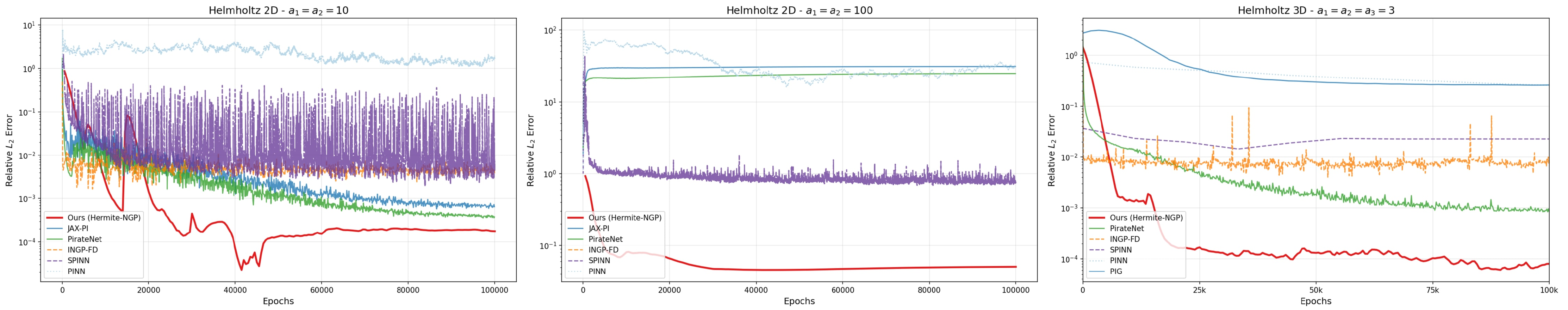}
    \caption{Training loss curves across three benchmarks. The horizontal axis indicates the number of epochs, while the vertical axis reports the relative $\ell_2$ training error. Our method exhibits a smoothly decreasing loss trajectory, in contrast to the strongly oscillatory behavior of the baselines, and converges to training errors up to \textbf{four orders of magnitude} smaller across the benchmarks; see Table~\ref{tab:main_results} for details.}
    \label{fig:loss_curves}
\end{figure}


\begin{figure}[H]
    \centering
    \includegraphics[width=1.0\linewidth]{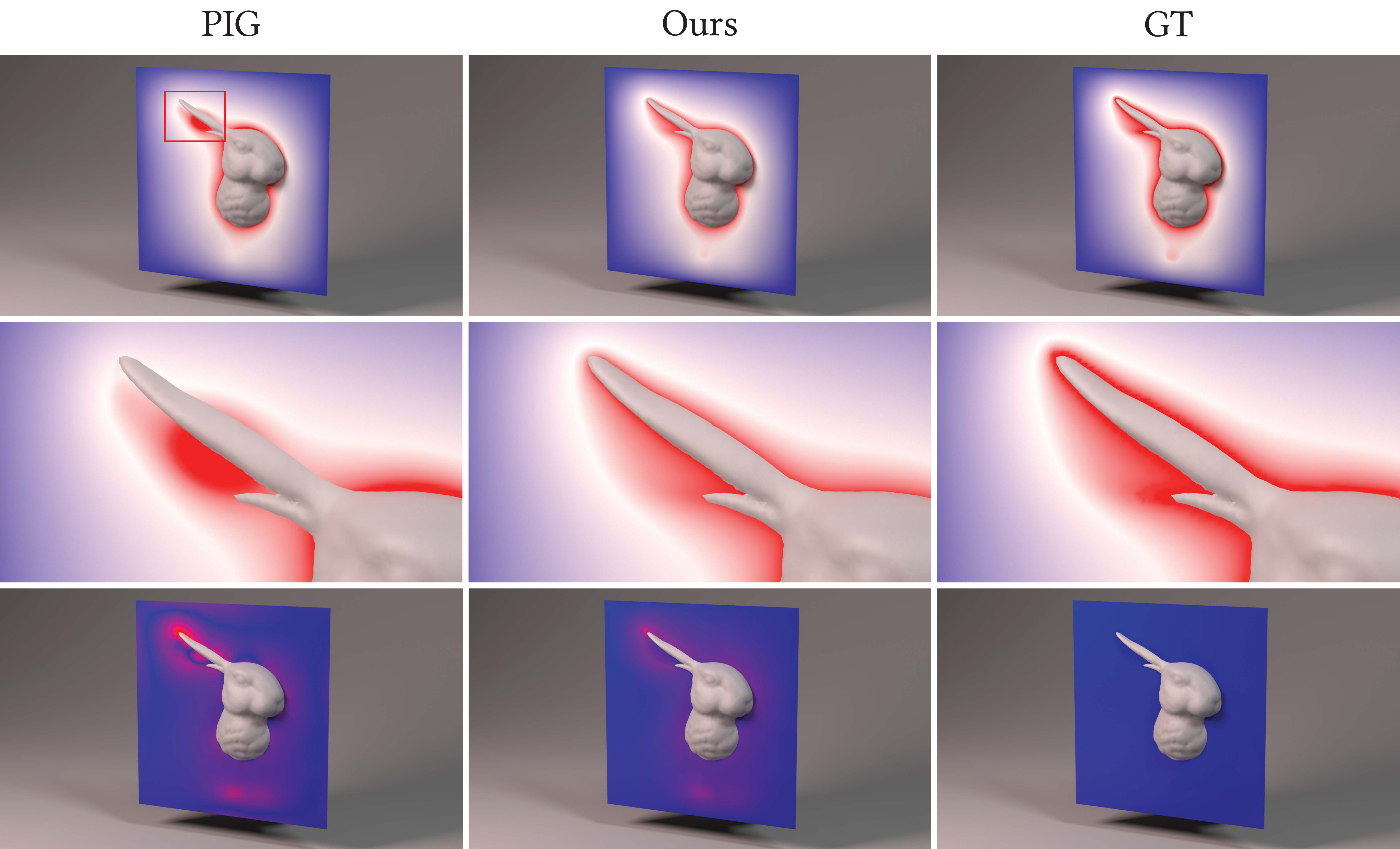}
    \caption{\textbf{Poisson 3D (Bunny).} Cross-sectional slice with mesh boundary and zoom-in detail. Hermite-NGP
  achieves MAE 0.0044 vs.\ PIG's 0.0127 (2.9$\times$ lower), with more accurate field near the boundary.}
    \label{fig:bunny_poisson}
\end{figure}

\begin{figure}[H]
    \centering
    \includegraphics[width=1.0\linewidth]{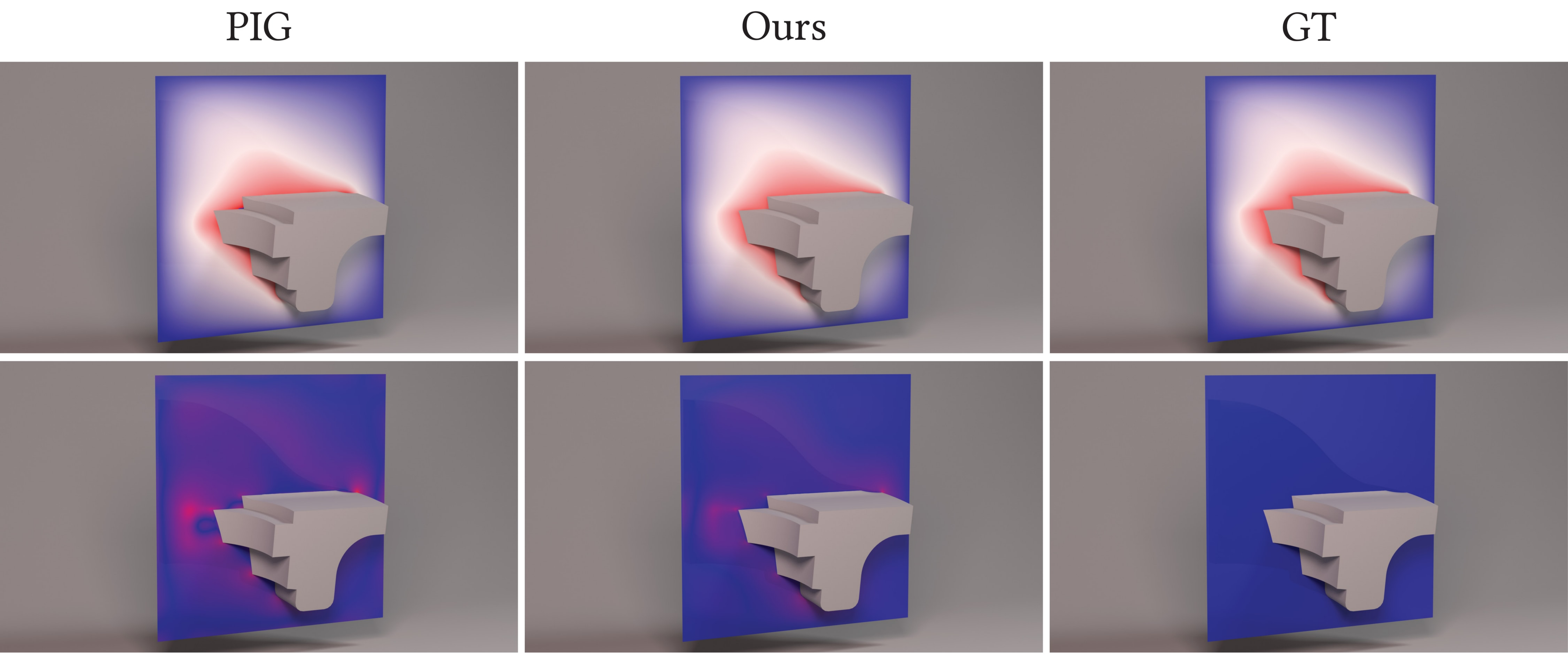}
    \caption{\textbf{Poisson 3D (Fandisk).} Cross-sectional slice with mesh boundary. Hermite-NGP achieves MAE 0.0031
  vs.\ PIG's 0.0100 (3.2$\times$ lower), accurately capturing sharp geometric features.}
    \label{fig:fandisk_poisson}
\end{figure}

\begin{figure}[H]
    \centering
    \includegraphics[width=1.0\linewidth]{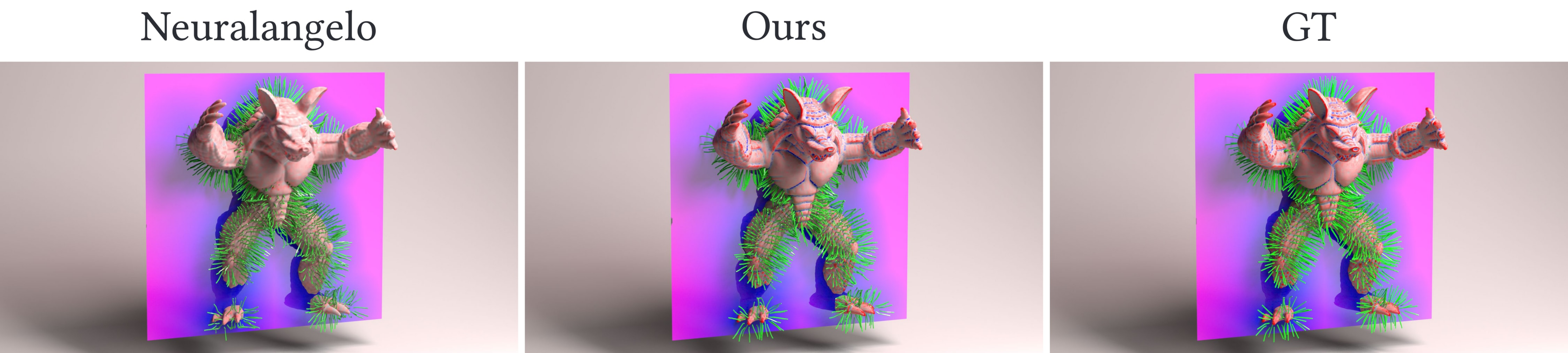}
    \caption{\textbf{SDF Curvature (Armadillo).} Mesh colored by mean curvature, green lines show gradient direction.
  Hermite-NGP achieves gradient MAE 0.0478 vs.\ NeuralAngelo's 0.1009 (2.1$\times$ lower).}
    \label{fig:arma_sdf_curvature}
\end{figure}

\begin{figure}[H]
    \centering
    \includegraphics[width=1.0\linewidth]{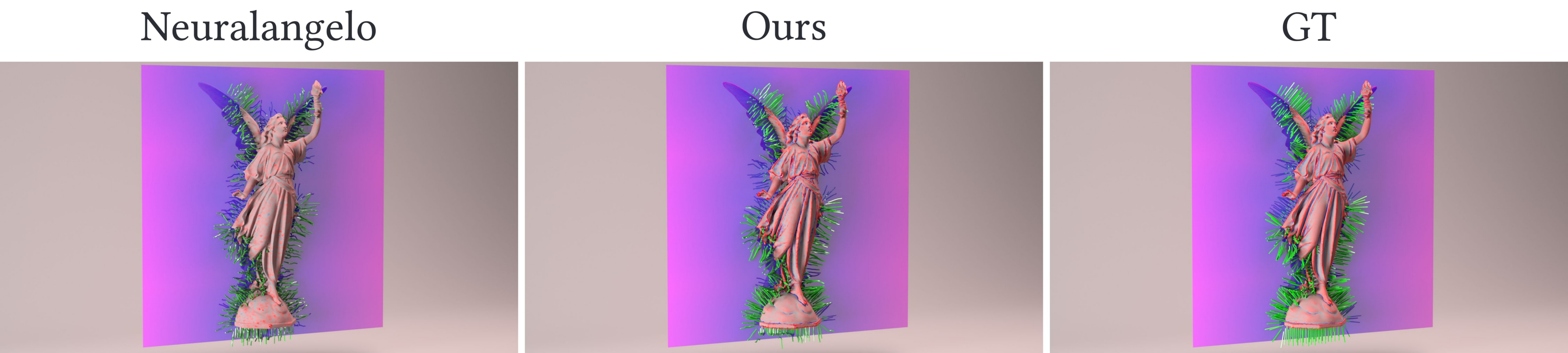}
    \caption{\textbf{SDF Curvature (Lucy).} Mesh colored by mean curvature, green lines show gradient direction.
  Hermite-NGP achieves gradient MAE 0.0418 vs.\ NeuralAngelo's 0.1213 (2.9$\times$ lower).}
    \label{fig:arma_sdf_curvature}
\end{figure}

\begin{figure}[H]
    \centering
    \includegraphics[width=1.0\linewidth]{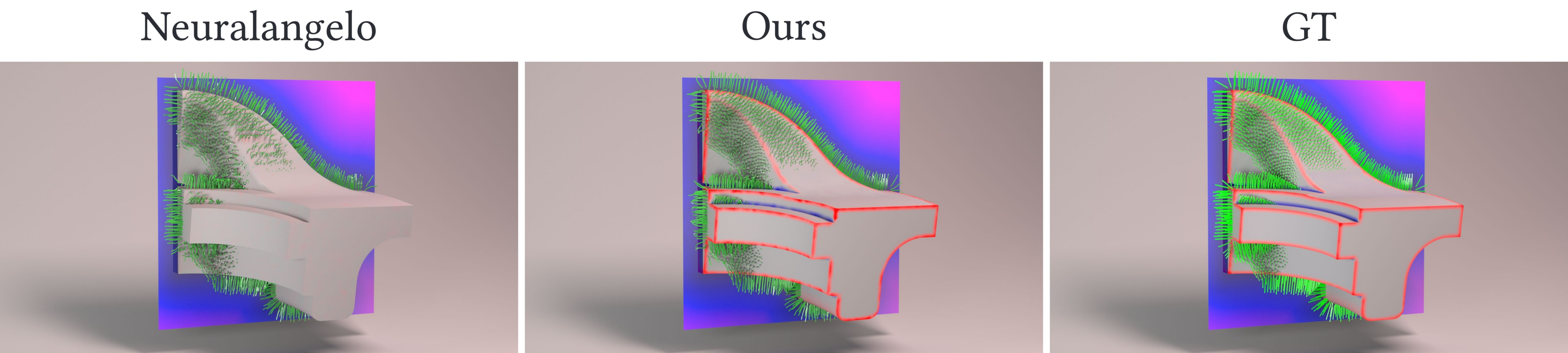}
    \caption{\textbf{SDF Curvature (Fandisk).} Mesh colored by mean curvature, green lines show gradient direction.
  Hermite-NGP achieves gradient MAE 0.0516 vs.\ NeuralAngelo's 0.1064 (2.1$\times$ lower), preserving sharp edges.}
    \label{fig:fandisk_sdf_curvature}
\end{figure}

\begin{figure}[H]
    \centering
    \includegraphics[width=1.0\linewidth]{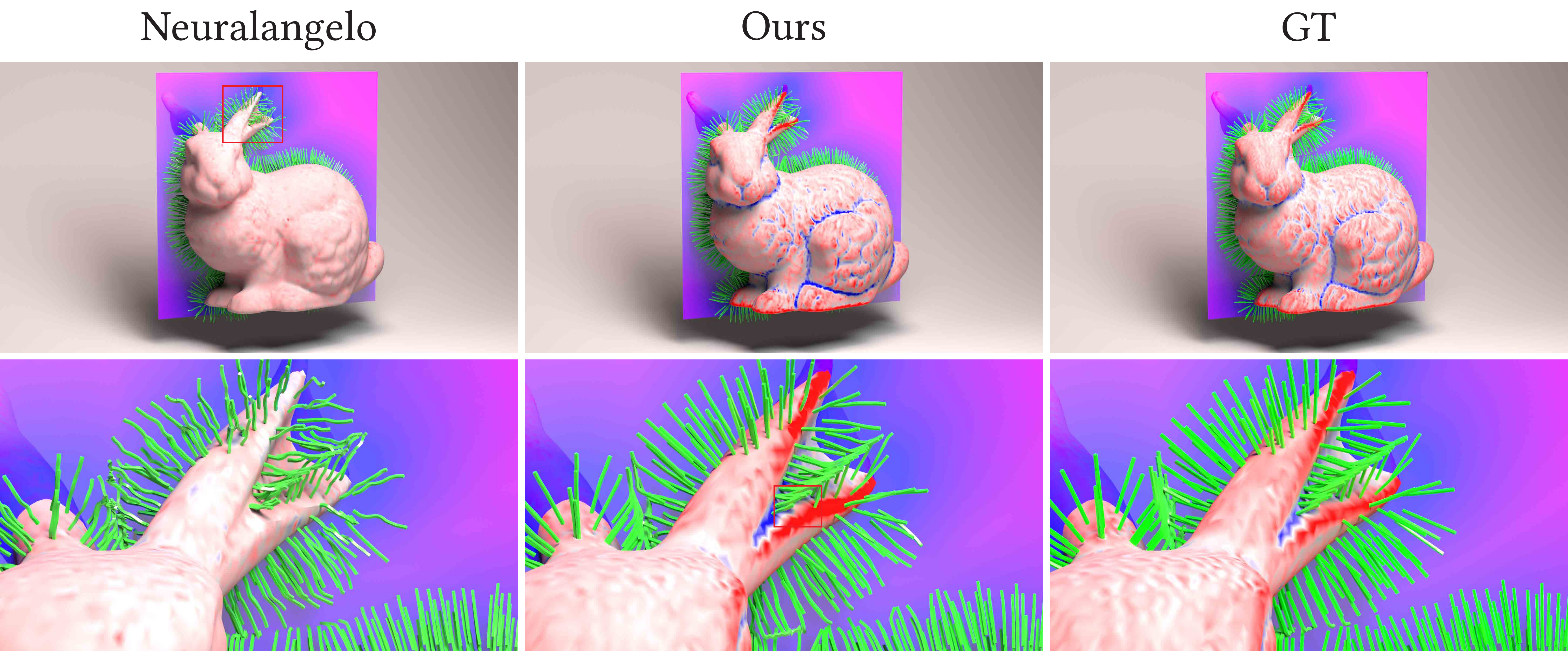}
    \caption{\textbf{SDF Curvature (Bunny).} Mesh colored by mean curvature, green lines show gradient direction.
  Hermite-NGP achieves gradient MAE 0.0416 vs.\ NeuralAngelo's 0.0887 (2.1$\times$ lower).}
    \label{fig:arma_sdf_curvature}
\end{figure}


\end{document}